%% file: main.tex
\documentclass{article}

\usepackage{microtype}
\usepackage{graphicx}
\usepackage{subfigure}
\usepackage{enumitem}
\usepackage{booktabs} 
\usepackage{hyperref}

\usepackage[accepted]{icml2025}
\usepackage{float}
\usepackage{amsmath}
\usepackage{amssymb}
\usepackage{mathtools}
\usepackage{amsthm}
\usepackage{tabularx}
\usepackage{makecell}
\usepackage[capitalize,noabbrev]{cleveref}
\usepackage{tablefootnote}
\usepackage{threeparttable}
\theoremstyle{plain}
\newtheorem{theorem}{Theorem}[section]

\newtheorem{lemma}[theorem]{Lemma}

\theoremstyle{definition}
\newtheorem{definition}[theorem]{Definition}

\theoremstyle{remark}

\usepackage[textsize=tiny]{todonotes}

\icmltitlerunning{Generalization Performance of Ensemble Clustering: From Theory to Algorithm}

\begin{document}

\twocolumn[
\icmltitle{Generalization Performance of Ensemble Clustering: From Theory to Algorithm}



\begin{icmlauthorlist}
    \icmlauthor{Xu Zhang}{seu}
    \icmlauthor{Haoye Qiu}{seu}
    \icmlauthor{Weixuan Liang}{nudt}
    \icmlauthor{Hui Liu}{SFU}
    \icmlauthor{Junhui Hou}{cityU}
    \icmlauthor{Yuheng Jia}{seu,SFU,seu2}
\end{icmlauthorlist}

\icmlaffiliation{seu}{School of Computer Science and Engineering, Southeast University, Nanjing 210096, China}
\icmlaffiliation{seu2}{Key Laboratory of New Generation Artificial Intelligence Technology and Its Interdisciplinary Applications (Southeast University), Ministry of Education, China}
\icmlaffiliation{nudt}{College of Computer Science and Technology, National University of Defense Technology, Changsha, China}
\icmlaffiliation{SFU}{School of Computing Information Sciences, Saint Francis University, Hong Kong, China}
\icmlaffiliation{cityU}{Department of Computer Science, City University of Hong Kong, Hong Kong, China}

\icmlcorrespondingauthor{Yuheng Jia}{yhjia@seu.edu.cn}

\icmlkeywords{Ensemble Clustering, Infinite Ensemble, CA matrix, Generalization Performance}

\vskip 0.3in
]



\printAffiliationsAndNotice{} 

\begin{abstract}
Ensemble clustering has demonstrated great success in practice; however, its theoretical foundations remain underexplored. This paper examines the generalization performance of ensemble clustering, focusing on generalization error, excess risk and consistency. We derive a convergence rate of generalization error bound and excess risk bound both of $\mathcal{O}(\sqrt{\frac{\log n}{m}}+\frac{1}{\sqrt{n}})$, with $n$ and $m$ being the numbers of samples and base clusterings. Based on this, we prove that when $m$ and $n$ approach infinity and $m$ is significantly larger than log $n$, i.e., $m,n\to \infty, m\gg \log n$, ensemble clustering is consistent. Furthermore, recognizing that $n$ and $m$ are finite in practice, the generalization error cannot be reduced to zero. Thus, by assigning varying weights to finite clusterings, we minimize the error between the empirical average clusterings and their expectation. From this, we theoretically demonstrate that to achieve better clustering performance, we should minimize the deviation (bias) of base clustering from its expectation and maximize the differences (diversity) among various base clusterings. Additionally, we derive that maximizing diversity is nearly equivalent to a robust (min-max) optimization model. Finally, we instantiate our theory to develop a new ensemble clustering algorithm. Compared with SOTA methods, our approach achieves average improvements of 6.1\%, 7.3\%, and 6.0\% on 10 datasets w.r.t. NMI, ARI, and Purity. The code is available at https://github.com/xuz2019/GPEC.
\end{abstract}

\input{Contents/Introduction}
\input{Contents/Preliminaries}
\input{Contents/Proposed}
\input{Contents/Diversity}

\input{Contents/Algorithm}
\input{Contents/Experiment}
\input{Contents/Conclusion}

\bibliography{references}
\bibliographystyle{icml2025}

\newpage
\appendix
\onecolumn
\input{Contents/Appendics}

\end{document}

%% file: Contents/Introduction.tex
\section{Introduction}
Ensemble clustering has attracted great attention in recent years due to its high accuracy and robustness compared to single clustering algorithm. It integrates multiple clustering results to obtain a consensus one instead of the access to the original features of the data, making it broadly applicable across various scenarios \cite{strehl2002cluster}. Many scholars have made considerable efforts in this area. For example, Fred and Jain \cite{fred2005combining} utilized a voting mechanism to generate an $n\times n$ similarity matrix to describe the relationships between sample pairs ($n$ is the number of samples), and applied hierarchical clustering to derive the final clustering results. Huang \cite{huang2017locally} realized that the importance of clusters in ensemble pool varies and assigned different weights to various clusters by estimating their uncertainty. Recently, Jia \cite{jia2023ensemble} utilized the high-confidence relationships to propagate similarity and designed a self-enhancement framework for the similarity matrix. More researches on ensemble clustering can be found in \cite{topchy2005clustering,jia2019semi,6413733,jia2021multi,zhang2022weighted,10238807,Xu_Li_Duan_2024, li2025conmix, peng2023egrc}.

Despite significant advances in practice, the theoretical analysis of ensemble clustering remains far from satisfactory. Theoretical analysis of an algorithm helps us understand its generalization performance such as generalization error, excess risk and consistency. Generalization error represents the expected loss of an algorithm across the entire data distribution. Excess risk refers to the difference between the expected loss of a model and the expected loss of the optimal model. For consistency, it means that whether a learning algorithm can uncover the true underlying structure of the data as the amount of training data increases. Most previous studies \cite{10.1214/aos/1176345339, pmlr-v70-bachem17a, Li_Ouyang_Liu_2023} focus on the generalization performance of a single clustering algorithm. To the best of our knowledge, only one paper \cite{7811216} has established a generalization error bound in the field of ensemble clustering while the excess risk and consistency are neglected. It demonstrates, from the perspective of weighted kernel $k$-means, that the generalization error bound of ensemble clustering is $\mathcal{O}(1/\sqrt{n})$. However, this work fails to consider that each base clustering should be treated as a random variable, which makes this study fundamentally no different from the researches of the generalization error of a single clustering algorithm. In ensemble clustering, we should consider not only the distribution of the data but also the distribution of the base clusterings. This underscores the need to understand the relationship between the number of samples $n$ and the number of base clusterings $m$. Therefore, in this paper, we investigate the generalization error bound and excess risk bound for ensemble clustering, and get the conclusion that both of them are $\mathcal{O}(\sqrt{\frac{\log n}{m}}+\frac{1}{\sqrt{n}})$. Based on these results, we derive the sufficient conditions for the consistency of ensemble clustering: both $m$ and $n$ approach infinity and $m$ is significantly larger than $\log n$, i.e., $m,n\to \infty, m \gg \log n$.

Although the above conclusion reveals the relationship between $m$ and $n$ in ensemble clustering, it is impractical in the real world to actually acquire infinite sample points and base clusterings. Therefore, we further consider whether it is possible to reduce the loss between the empirical average of base clusterings and the expectation of base clustering. By deriving the loss function between them, we reveal that minimizing the deviation of each base clustering with its expectation (bias) and maximizing the differences among various base clusterings (diversity) can promote the clustering performance. However, once the base clusterings are given, both the bias and diversity are fixed. We, therefore, transform ensemble clustering into a learnable problem by weighting the base clusterings to decrease the loss, from which we also find that maximizing diversity is nearly equivalent to a robust optimization problem. By instantiating our theory, we design a new ensemble clustering algorithm and optimize it by the reduced gradient descent method. In summary, the key contributions of this work are:
\begin{itemize}
    \item We pioneer the derivation of the generalization error bound and excess risk bound for ensemble clustering, incorporating considerations of both data and clustering distributions. We also establish sufficient conditions for the consistency of ensemble clustering, a novel advancement in the field.
   \item Our theoretical exploration uncovers that in ensemble clustering, minimizing bias between each base clustering and its expectation, alongside maximizing diversity among base clusterings, enhances clustering performance. Moreover, we establish a fundamental link between diversity and robustness in this context.
   \item Building upon our theoretical framework, we introduce a novel ensemble clustering algorithm and address it through the reduced gradient descent method, offering a practical solution based on rigorous theoretical underpinnings.
    \item  By extensive experimental validation, we confirm the validity of our theoretical assertions and demonstrate that the proposed algorithm surpasses other state-of-the-art methods significantly in terms of performance.
\end{itemize}

%% file: Contents/Preliminaries.tex
\section{Preliminaries}\label{Preliminaries}
In this section, we first briefly introduce key notations and general assumptions. Some of these align with \cite{von2008consistency} and \cite{pmlr-v202-liang23b}, where readers can consult for further details. We then proceed to describe the co-association matrix in ensemble clustering.
\subsection{Notations and General Assumptions}
In this paper, let $n$ represent the number of samples and $m$ the number of base clusterings. $(\cdot)^\top$ and $\mathrm{tr}(\cdot)$ are used to transpose and calculate the trace of a matrix. $||\mathbf{A}||_\mathrm{F}$ is the Frobenius norm of a matrix. $||\mathbf{A}||_2$ donates the spectral norm of a matrix $\mathbf{A}$, $||\mathbf{a}||_2$ is $\ell_2$-norm for vector $\mathbf{a}$. $\mathbf{A} \preceq \mathbf{B}$ means $\mathbf{B}-\mathbf{A}$ is positive semi-definite.

We assume the sample space $\mathcal{X}$ is compact. Let $\rho (x)$ and $\rho_n(x)$ denote the corresponding true probability distribution and empirical distribution of $x$, respectively. The dataset $\mathrm{S}_n=\{x_1,\cdots,x_n\}$ is  collected independently and identically distributed (i.i.d.) from $\mathcal{X}$ according to the distribution $\rho$. We denote $\pi^{(t)}$ as a base clustering generated i.i.d. by a clustering algorithm. $\pi^{(t)}(x_i)$ is the clustering label of the $t$-th base clustering for data $x_i$. We denote $\pi^{(t)}$ as an $n\times 1$ vector and $k^{(t)}$ as the number of clusters for $\pi^{(t)}$. $\Pi=\{\pi^{(1)},\cdots,\pi^{(m)}\}$ is the ensemble base clustering pool with $m$ base clusterings.
\vspace{-0.5em}
\subsection{Co-Association Matrix} \label{CA}
In clustering, as no supervision is available, the labels we obtain are not aligned with the true labels of the samples. Nonetheless, the similarity relationship between sample pair is unique, we can define the similarity for each base clustering $\pi^{(t)}$ uniquely as 
\vspace{-0.5em}
\begin{equation*}
    \mathbf{A}^{(t)}_{ij} = \delta(\pi^{(t)}(x_i),\pi^{(t)}(x_j)), \    \delta \left( a,b \right) =\begin{cases}
	1,&		\text{if} \ \ a=b,\\
	0,&		\text{else}.\\
    \end{cases}
    \vspace{-1em}
\end{equation*}

The CA matrix $\bar{\mathbf{A}}$ \cite{fred2005combining} is the average of these similarity matrices, $\bar{\mathbf{A}}=\frac{1}{m}\sum_{t=1}^m \mathbf{A}^{(t)}$. Since each similarity matrix $\mathbf{A}^{(t)}$ is a positive semi-definite matrix and $\bar{\mathbf{A}}$ is a convex combination of these matrices, the CA matrix is also positive semi-definite. CA-based ensemble clustering methods \cite{7337436,8525437,zhou2023partial,ji2024clustering} try to learn a more accurate CA matrix, and then perform hierarchical clustering or spectral clustering on it to obtain a more accurate consensus result.

%% file: Contents/Proposed.tex
\section{Generalization Performance}\label{generalization_per}
Based on the definition in Section \ref{CA}, we define the degree normalized similarity matrix $\mathbf{K}^{(t)}$ of $\mathbf{A}^{(t)}$ is $\mathbf{K}^{(t)} = \mathbf{D}^{(t)-1/2}\mathbf{A}^{(t)}\mathbf{D}^{(t)-1/2}$ where $\mathbf{D}^{(t)-1/2}$ is the degree matrix of $\mathbf{A}^{(t)}$. Obviously $\mathbf{K}^{(t)}$ is still symmetric and positive semi-definite and we assume $\mathbf{K}^{(t)}$ is generated from a kernel function 
$K^{(t)}$, where $\mathbf{K}^{(t)}_{ij}=K^{(t)}(x_i,x_j)$. The empirical error function $\hat{F}\left( \hat{\mathbf{Z}};\bar{\mathbf{K}} \right)$ for ensemble clustering is defined as:
\begin{equation}\label{epirical}
    \hat{F}\left( \hat{\mathbf{Z}};\bar{\mathbf{K}} \right) =\frac{1}{n}\underset{\hat{\mathbf{Z}}\in \mathbb{R}^{n\times k}}{\max}\,\,\text{tr}\left( \hat{\mathbf{Z}}^{\top}\mathbf{\bar{K}}\hat{\mathbf{Z}} \right) , \text{s}.\text{t}.\ \hat{\mathbf{Z}}^{\top}\hat{\mathbf{Z}}=\mathbf{I},
\end{equation}
where $\hat{\mathbf{Z}}$ represents the spectral embedding of (normalized) CA matrix $\bar{\mathbf{K}}$, which is utilized to approximate the cluster indicator matrix. $\bar{\mathbf{K}}=\frac{1}{m}\sum_{t=1}^m\mathbf{K}^{(t)}$ is the average of normalized similarity matrices, and the coefficient $\frac{1}{n}$ in Eq. (\ref{epirical}) guarantees the convergence of eigenvalues of the kernel matrix to those of the corresponding integral operator as $n\to \infty$ \cite{10496870,JMLR:v11:rosasco10a}. Let $\{\hat{\lambda}_q\}_{q=1}^k$ be the largest $k$ eigenvalues of $\frac{1}{n}\bar{\mathbf{K}}$. The solution to Eq. (\ref{epirical}) is the eigenvectors $\hat{\mathbf{Z}}=[\mathbf{z}_1,\cdots,\mathbf{z}_k]$ corresponding to $k$ largest eigenvalues of $\bar{\mathbf{K}}$. Considering the true continuous distribution of the data, we define the following integral operator $L_Kg\left( x \right): L^2\left( \mathcal{X},\rho \right) \rightarrow \,\,L^2\left( \mathcal{X},\rho \right)$
\begin{equation*}
    L_Kg\left( x \right) =\int_{\mathcal{X}}{K\left( x,y \right) g\left( y \right)}\text{d}\rho \left( y \right) ,
\end{equation*}
where $L^2$ denotes square-integrable function space. According to the definition of eigenfunction, we have
\begin{equation*}
    \zeta_q\left( x \right) =\frac{1}{\lambda _p}\int_{\mathcal{X}}{K\left( x,y \right)}\zeta_q\left( y \right) \text{d}\rho \left( y \right) ,
\end{equation*}
where $\zeta_q(x)$ is the corresponding eigenfunction of $\lambda_q$, and $\lambda_q$ is the eigenvalue of $L_K$. Thus, we define the error measured over the entire distributions of data and base clusterings, referred to as the population-level error with the expectation of base clustering,
\begin{equation}
    \begin{aligned}
        &F( \mathcal{Z};K^*) =
        \\
        &\max_{\left\{ \zeta_q \right\} _{q=1}^{k}\in \varGamma} \sum_{q=1}^k{\iint_{\mathcal{X}}{K^*( x,y ) \zeta_q( x ) \zeta_q( y )}}\text{d}\rho ( x ) \text{d}\rho ( y ) ,
    \end{aligned}
\end{equation}
where $\mathcal{Z}=\{\zeta_q\}_{q=1}^k$ denotes the corresponding eigenfunctions of integral operator $L_{K^*}$ with eigenvalues $\{\lambda_q\}_{q=1}^k$.  $K^*(x,y)=\mathbb{E}[K^{(t)}(x,y)]$ is the expectation of the normalized similarity function $K^{(t)}$. Note that 
$\mathbb{E}[\bar{K}]=\mathbb{E}[K^{(t)}]=K^*$, meaning the expectation of the CA function ($\bar{K}$) is the same as that of a single normalized similarity function. In the following sections, we will sometimes refer to $K^*$ as the expectation of the CA function.

However, as $\hat{\mathbf{Z}}$ and $\mathcal{Z}$ lie in the different space, we define the empirical integral operator, which is the approximation of the theoretical integral operator based on finite samples, $\hat{L}_K: L^2(\mathcal{X},\rho_n)\to L^2(\mathcal{X},\rho_n)$ as
\begin{equation*}
    \setlength{\abovedisplayskip}{3pt}
    \setlength{\belowdisplayskip}{3pt}
    \hat{L}_K\hat{z}_q\left( x \right) =\frac{1}{n}\sum_{i=1}^n{K\left( x,x_i \right) \hat{z}_q\left( x_i \right)}.
\end{equation*}
According to \cite{6788387}, the eigenvalues of $\frac{1}{n}\bar{\mathbf{K}}$ and $\hat{L}_{\bar{K}}$ are the same except zero eigenvalues, and the empirical eigenfunctions of $\frac{1}{n}\bar{\mathbf{K}}$ are
\begin{equation*}
    \hat{z}_q(x)=\frac{1}{n\hat{\lambda}_q} \sum_{i=1}^n \bar{K} (x,x_i)\hat{z}_q(x_i),
\end{equation*}
where $\hat{z}_q(x_i) = \sqrt{n} \mathbf{z}_{iq}$. Thus, Eq. (\ref{epirical}) is rewritten as
\begin{equation*}
    \hat{F}( \hat{Z};\bar{K} ) =\underset{\left\{ \hat{z}_q \right\} _{q=1}^{k}}{\max}\frac{1}{n^2}\sum_{q=1}^k{\sum_{i=1}^n{\sum_{j=1}^n{\bar{K}( x_i,x_j ) \hat{z}_q( x_i ) \hat{z}_q ( x_j )}}}.
\end{equation*}
\textbf{Key problems}: According to the above definitions, we investigate the generalization performance of ensemble clustering including generalization error bound, excess risk bound, and sufficient conditions for consistency, which are defined as follows:
\begin{itemize}[noitemsep,topsep=-1pt]
    \item Generalization error: the difference between empirical error and population-level error, represented as $\hat{F}(\hat{{Z}};\bar{{K}})-F(\mathcal{Z};K^*)$; 
    \item Excess risk:  quantifying the difference in error between a learning algorithm and the optimal algorithm on data distribution, expressed as $F(\hat{{Z}};K^*)-F(\mathcal{Z};K^*)$;
    \item Consistency: the clusterings produced by the given algorithm converge to a clustering that represents the entire underlying space. That is, as the number of samples $n$ and base clusterings $m$ increase, the empirical eigenvectors $\hat{\mathbf{Z}}$ converge to the eigenfunctions $\mathcal{Z}$ of the true underlying structure.
\end{itemize}
The following three theorems address the key problems.
\begin{theorem}\label{generalization}
Under the general assumptions and assume that the gap between the $k$-th and $(k + 1)$-th eigenvalues of the expectation of normalized similarity matrix $\mathbf{K}^*$ is $\delta_k$ and $\delta_k\ge \frac{1}{c}>0$ where $c$ is a constant. For any $0<\delta < 1$, with probability at least $1-\delta$, we have
{\small
\begin{align}
    &\hat{F}\left( \hat{Z};\bar{K} \right) -F\left( \mathcal{Z} ;K^* \right) 
    \\
    \le & \left(2\sqrt{2}c+1\right)\left(\frac{2}{3m}\log \frac{6n}{\delta}+\sqrt{\frac{8}{m}\log \frac{6n}{\delta}}\right)+\frac{2\sqrt{2}\log \left(\frac{6}{\delta}\right)}{\sqrt{n}} . \nonumber
\end{align}
}
\end{theorem}
\begin{proof}
See Appendix \ref{proof_gen}.
\end{proof}
\textbf{Remark}. Theorem \ref{generalization}, for the first time, presents the generalization error bound of ensemble clustering under the consideration of both the data distribution and the base clustering distribution, which is $\mathcal{O}(\sqrt{\frac{\log n}{m}}+\frac{1}{\sqrt{n}})$. Through this theorem, we establish the relationship between the sample size $n$ and the number of base clusterings $m$. Clearly, if sample size $n$ is fixed, the generalization error continues to decrease as the number of base clusterings increases, although it will not converge to zero. However, with a fixed $m$, we cannot guarantee the decrease of generalization error, instead, it tends to infinity as $n$ increases. Thus, \textit{in ensemble clustering, simply acquiring more samples is not an effective strategy, we still need to obtain more base clusterings as data size increases}. Additionally, a rapid growth of $m$ is required to allow the generalization error to converge to $0$, which implies that $m$ should be significantly larger than $\log n$ (i.e., $m \gg \log n$). In practice, we recommend setting $m=\sqrt{n}$ to strike a balance between theoretical convergence and computational efficiency of time and space.
\begin{theorem}\label{excess}
   Under the same assumptions as Theorem \ref{generalization} and with the additional condition that $||\hat{z}_q ||_{\infty} \le \sqrt{c_0}$ ($c_0>0$ is a constant), for any $0<\delta<1$, with probability at least $1-\delta$, we have
{\small
\begin{align}
    &F\left( \hat{Z};K^* \right) -F\left( \mathcal{Z};K^* \right) \le k\left(\frac{2\sqrt{2}c_0}{\sqrt{n}}+\sqrt{\frac{8\log \frac{3}{\delta}}{n}}\right) \nonumber
    \\
    +&2\sqrt{2}c\left(\frac{2}{3m}\log \frac{6n}{\delta}+\sqrt{\frac{8}{m}\log \frac{6n}{\delta}}\right)+\frac{2\sqrt{2}\log \left( \frac{6}{\delta} \right)}{\sqrt{n}}. 
\end{align}
}
\end{theorem}
\begin{proof}
See Appendix \ref{proof_excess}.
\end{proof}
\textbf{Remark}. 
Theorem \ref{excess} provides the excess risk bound for ensemble clustering, which is also expressed as $\mathcal{O}(\sqrt{\frac{\log n}{m}}+\frac{1}{\sqrt{n}})$. Obviously, we require the same condition as in Theorem \ref{generalization} (i.e., $m,n\to \infty, m\gg \log n$) to ensure that the population-level error ($F(\hat{Z};K^*)$) of the learned algorithm ($\hat{Z}$) converges to that  ($F(\mathcal{Z};K^*)$) of the optimal algorithm ($\mathcal{Z}$) on the entire data and clustering distributions. It is worth noting that this theorem introduces an additional mild assumption $||\hat{z}_q||_\infty \le \sqrt{c_0}$, which is easily satisfied given that $\bar{K}(x,x_i)\le 1$, $\sum_{i=1}^n\hat{z}_q(x_i)\le {n}$ and $\hat{z}_q(x)\le \frac{1}{\hat{\lambda}_q}$.

\begin{theorem}\label{consistency}
    Under the same assumptions as Theorem \ref{generalization}, if $m,n \to \infty$ and ${\lim_{{m,n\rightarrow \infty}}}\,\,\frac{\log n}{m}\rightarrow 0$, there exists a sequence $(a_q)_q \in \{-1,1\}$ such that 
    \begin{equation*}
        \lVert a_q\mathbf{\hat{z}}_q-\zeta _q \rVert _{\infty} \to 0,
    \end{equation*}
    in probability.
\end{theorem}
\begin{proof}
See Appendix \ref{proof_consistency}.
\end{proof}
\textbf{Remark}.
Theorem \ref{consistency} provides the sufficient conditions for the 
 consistency of ensemble clustering. It describes that the corresponding empirical eigenvectors converge to the eigenfunctions in the limit case. Based on this, we conclude that the clustering learned from empirical data can converge to the true underlying structure of the data, thereby ensuring the consistency of ensemble clustering. Note that since multiplying the eigenvectors by $\pm 1$ does not affect the outcome, we need to prepend a coefficient $a_q$ to $\hat{\mathbf{z}}_q$ to ensure that the signs of $\hat{\mathbf{z}}_q$ and $\zeta_q$ are consistent.

%% file: Contents/Diversity.tex
\section{Key Factors in Ensemble Clustering}\label{Bias_Diversity_Decom}
While the preceding section offers theoretical guarantees for the performance of ensemble clustering when both $m$ and $n$ approach infinity, and explores the relationship between $m$ and $n$, it is not feasible to obtain infinite data points and base clusterings in practice. Therefore, in this section, we consider how to approximate the expectation of clustering ($\mathbf{K}^*$) with the average of the finite base clusterings (the CA matrix $\bar{\mathbf{K}}$) by the following optimization problem 
\begin{equation}\label{loss}
    \begin{aligned}
        \min  \mathcal{L}&=\hat{F}\left( \mathbf{\hat{Z}};\mathbf{\bar{K}} \right) -\hat{F}\left( \hat{\mathcal{Z}};\mathbf{K}^* \right) 
        \\
        &=\frac{1}{n}\text{tr}\left( \mathbf{\hat{Z}}^{\top}\mathbf{\bar{K}\hat{Z}} \right) -\frac{1}{n}\text{tr}\left( \hat{\mathcal{Z}}^{\top}\mathbf{K}^*\hat{\mathcal{Z}} \right) .
    \end{aligned}
\end{equation}
When $\mathcal{L}=0$, we perfectly fit the underlying structure of the samples using a finite number of base clusterings. However, once the base clusterings are established, $\bar{\mathbf{K}}$ is fixed and so as to the associated $\mathcal{L}$. To decrease the loss $\mathcal{L}$, we apply different weights to various base clusterings. Accordingly, we substitute CA matrix $\bar{\mathbf{K}}$ with weighted CA matrix $\mathbf{K}^\mathbf{w}$, which is defined as
\begin{align}\label{KW}
    \mathbf{K}^{\mathbf{w}}=\sum_{t=1}^m{w_t\mathbf{K}^{\left( t \right)}}.
\end{align}
We replace $\mathcal{L}$ as $\mathcal{L}^{\mathbf{w}}$ and obtain the follow theorem.
\begin{theorem}\label{BD_decom}
    Based on Eqs. (\ref{loss}) and (\ref{KW}) and let $c'={k}/{n}+{2\sqrt{2}}/\left({\lambda_k(\mathbf{K}^*)-\lambda_{k+1}(\mathbf{K}^*)}\right)$, $\lambda_k(\mathbf{K}^*)$ is the $k$-th eigenvalue of $\mathbf{K}^*$, $,\tilde{w}_t=mw_t$, $m$ is the number of base clusterings, we derive the Bias-Diversity decomposition for ensemble clustering, as
    \begin{small}
        \begin{align}\label{buzhidaodingyishenmelabel}
            &\ \ \ \ \ \ \ \ \ \ \ \ \min_{\mathbf{w}} \mathcal{L}^{\mathbf{w}}=\hat{F}\left( \mathbf{\hat{Z}};\mathbf{K}^{\mathbf{w}} \right) -\hat{F}\left( \hat{\mathcal{Z}};\mathbf{K}^* \right)
            \\
            \le & c'\sqrt{\frac{1}{m}( \underset{\mathrm{Bias}}{\underbrace{\sum_{t=1}^m{\lVert \tilde{w}_t\mathbf{K}^{( t )}-\mathbf{K}^* \rVert _{\mathrm{F}}^{2}}}}-\underset{\mathrm{Diversity}}{\underbrace{\sum_{t=1}^m{\lVert \tilde{w}_t\mathbf{K}^{( t )}-\mathbf{K}^{\mathbf{w}} \rVert _{\mathrm{F}}^{2}}}} )}. \nonumber
        \end{align}
    \end{small}
\end{theorem}
\textit{Proof}. See Appendix \ref{proof_theorem41}. \qed

\textbf{Remark}. This theorem describes the loss $\mathcal{L}^{\mathbf{w}}$ is governed by two terms: Bias and Diversity. Here, Bias describes the average gap between each single weighted base clustering ($\tilde{w}_t\mathbf{K}^{(t)}$) and the expectation of base clustering ($\mathbf{K}^*$), while Diversity describes the average difference between each single weighted base clustering ($\tilde{w}_t\mathbf{K}^{(t)}$) and the weighted CA matrix ($\mathbf{K}^\mathbf{w}$). Therefore, by adjusting $\mathbf{w}=\{w_t\}_{t=1}^m$ to achieve low Bias and high Diversity, we can reduce the loss $\mathcal{L}^\mathbf{w}$ and obtain better clustering performance.

To better analyze Theorem \ref{BD_decom}, we first simply Eq. (\ref{buzhidaodingyishenmelabel}) into a more concise from:
\begin{align}\label{simplify_eq}
        \underset{\mathbf{w}}{\min}\,\,&-2\mathrm{tr}\left( \mathbf{K}^{\mathbf{w}}\mathbf{K}^* \right) +\mathrm{tr}\left( \mathbf{K}^{\mathbf{w}}\mathbf{K}^{\mathbf{w}} \right) \nonumber
        \\
        &\quad \mathrm{s.t.}\ \mathbf{w}^\top \mathbf{w}=1, \mathbf{w} \ge 0.
\end{align}
where the constraint $\mathbf{w}^\top \mathbf{w}=1$ is imposed to avoid sparse solutions for the weights $\mathbf{w}$. 
The proof of Eq. (\ref{buzhidaodingyishenmelabel}) $\Rightarrow$ Eq. (\ref{simplify_eq}) is provided in Appendix \ref{proof_eq78}. Eq. (\ref{simplify_eq}) remains a non-convex optimization problem of $\mathbf{w}$ and we still need to process $\mathbf{K}^{\mathbf{w}}$ to obtain the final clustering results, such as performing hierarchical clustering or spectral clustering on it. To this end, we introduce the spectral embedding $\mathbf{Z}$ of $\mathbf{K}^\mathbf{w}$ to Eq. (\ref{simplify_eq}). Specifically. the first term of Eq. (\ref{simplify_eq}) ($\min_{\mathbf{w}}\,-2\text{tr}\left( \mathbf{K}^{\mathbf{w}}\mathbf{K}^* \right)$) is reformulated as $\max_{\mathbf{Z}} 2\mathrm{tr}(\mathbf{K}^*\mathbf{Z}\mathbf{Z}^\top)$ by substituting $\mathbf{K}^{\mathbf{w}}$ for $\mathbf{Z}\mathbf{Z}^\top$. For the second term $\min_{\mathbf{w}} \mathrm{tr}(\mathbf{K}^\mathbf{w}\mathbf{K}^\mathbf{w})$, we replace one instance of $\mathbf{K}^\mathbf{w}$ with the spectral embedding $\mathbf{Z}$, i.e., $\mathrm{tr}(\mathbf{K}^\mathbf{w}\mathbf{K}^\mathbf{w})\Rightarrow \max_{\mathbf{Z}} \mathrm{tr}(\mathbf{K}^\mathbf{w} \mathbf{Z}\mathbf{Z}^\top)$, and further transform it into a min-max optimization problem. Besides, the original constraint $\mathbf{w}^\top\mathbf{w}=1$ is non-convex, we revise it to $\mathbf{w}^\top \boldsymbol{1}=1$ (where $\boldsymbol{1}$ is a column vector of all ones) and also modify the definition of $\mathbf{K}^{\mathbf{w}}=\sum_{t=1}^m w_t^2 \mathbf{K}^{(t)}$, allowing $\mathbf{w}$ to be better interpreted as a weight distribution. Together with orthogonal constraint on the spectral embedding $\mathbf{Z}$, the optimization problem is finally redefined as: 
\begin{align}\label{main}
    \underset{\mathbf{Z}\in \mathbb{R}^{n\times k}}{\max}\,\,\underset{-\mathrm{Bias}}{\underbrace{2\mathrm{tr}\left( \mathbf{K}^*\mathbf{ZZ}^{\top} \right) }}&+\overset{\mathrm{Robust\ } \mathrm{optimization}}{\overbrace{\underset{\mathbf{w}}{\min}\underset{-\mathrm{Diversity}}{\underbrace{\max _{\mathbf{Z}\in \mathbb{R}^{n\times k}}\,\,\mathrm{tr}\left( \mathbf{K}^{\mathbf{w}}\mathbf{ZZ}^{\top} \right) }}}} \nonumber
    \\
    \mathrm{s.t.}\ \mathbf{Z}^{\top}\mathbf{Z}=\mathbf{I}, &\mathbf{w}^\top\boldsymbol{1}=1, \mathbf{w}\ge 0.
\end{align}
\textbf{Remark 1 (Diversity)}. From Eq. (\ref{main}), we observe that the Diversity term aims to enhance the diversity among the base clusterings. Although some heuristic methods \cite{fern2003random, kuncheva2006evaluation, HADJITODOROV2006264, JIA20111456, metaxas2023divclust} were proposed to increase diversity in ensemble clustering, our approach is entirely derived from Theorem \ref{BD_decom}, offering solid theoretical guarantees. 

\textbf{Remark 2 (Robust Optimization)}. We surprisingly discover that maximizing the Diversity term is equivalent to a robust min-max optimization model (aiming to identify the spectral embedding that performs well even with a bad weight vector), which is similar to some existing robust ensemble algorithm \cite{9857664, 10.1145/3503161.3547917, bang2018robust, ijcai2019p494, Liang_Liu_Zhou_Liu_Wang_Zhu_2022}. Unlike their motivation to enhance the model's resistance to noise, \textit{we explain that these algorithms essentially improve diversity within the ensemble to reduce the loss ($\mathcal{L}^\mathbf{w}$) between the empirical and expected error. This provides a theoretical explanations for why these algorithms work effectively}.

\textbf{Remark 3 (Bias)}. It is worth noting that existing algorithms \cite{bang2018robust, 9857664} only consider the optimization of diversity (robustness), while neglecting the Bias term. A natural concern arises for those methods: \textit{does min-max optimization sacrifice the most accurate individuals in the ensemble?} For example, if most individuals in the ensemble have high accuracy but a few have low accuracy, considering diversity might lead us to assign higher weights to the poorer performers, potentially dragging down the final consensus result (we will verify this in our experiments). Regrettably, existing algorithms neglect this issue. Our theory indicates that better clustering performance will be more likely to achieve by simultaneously optimizing (minimizing) bias and (maximizing) diversity in ensemble clustering.

%% file: Contents/Algorithm.tex
\section{Instantiation of Theorem \ref{BD_decom}}\label{Algorithm}
In this section, we instantiate our theoretical analysis (Eq. (\ref{main})) to obtain a novel ensemble clustering algorithm. Eq. (\ref{main}) is not directly usable as we do not know the true expected value of the CA matrix $\mathbf{K}^*$. Therefore, we try to approximate it using a simple yet effective way.
\subsection{Approximate $\mathbf{K}^*$}

\begin{figure*}
    \centering
    \subfigure[The proportion, precision, and recall of high-confidence elements in different threshold.]{
    \includegraphics[width=0.23\linewidth]{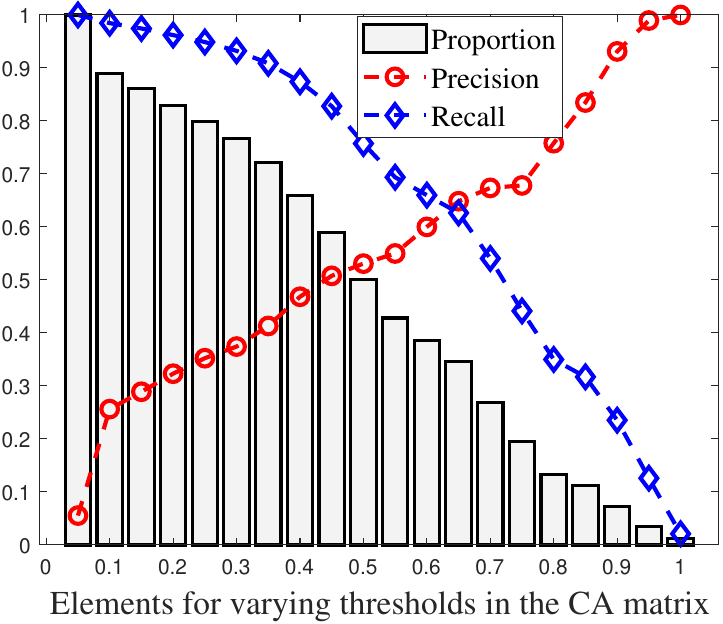}}
            \hspace{0.001\linewidth}
    \subfigure[High confidence matrix ${\mathbf{H}}$ with threshold $\alpha=0.4$.]{
    \includegraphics[width=0.225\linewidth]{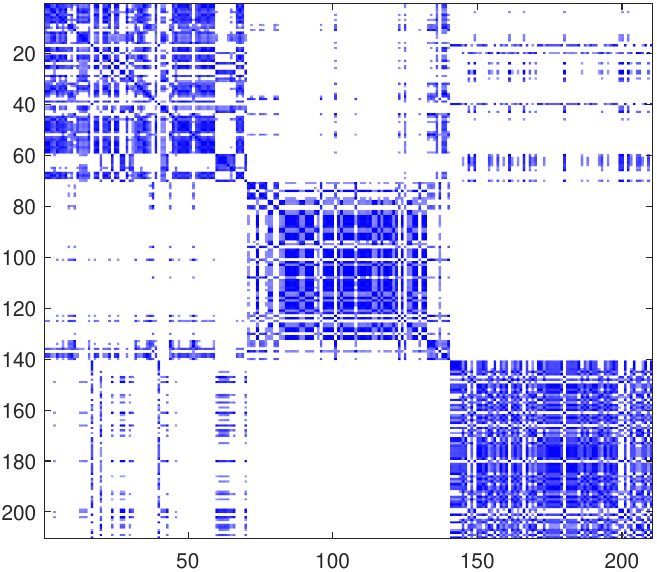}}
            \hspace{0.001\linewidth}
    \subfigure[Second-order similarity relation matrix $\tilde{\mathbf{K}}$ of $\mathbf{H}$.]{
    \includegraphics[width=0.225\linewidth]{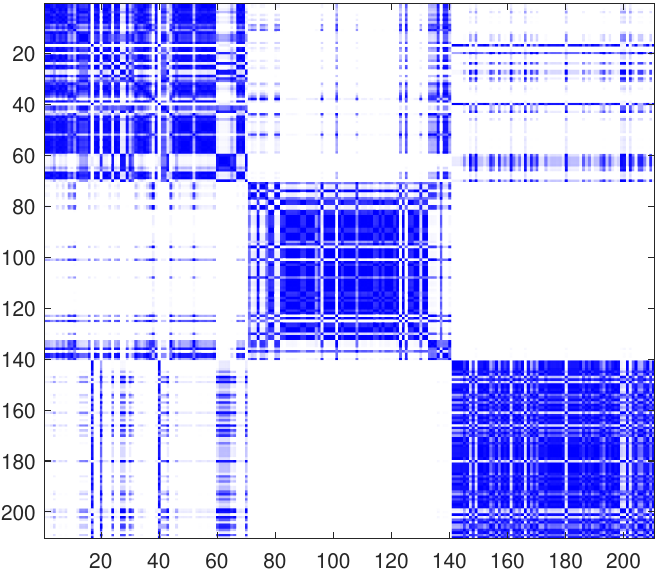}}
            \hspace{0.001\linewidth}
    \subfigure[The ground truth similarity matrix of seed dataset.]{
    \includegraphics[width=0.255\linewidth]{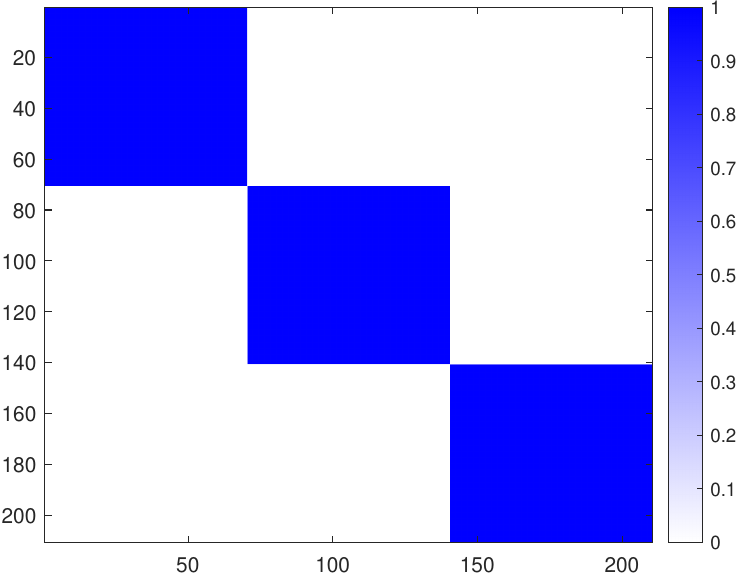}}
    \caption{As shown in Fig. (a), with the increase in the high-confidence threshold, the proportion of elements and the recall rate gradually decrease, but the precision approaches 1, suggesting that high-confidence elements are reliable. Fig. (b) displays the visualization of the high-confidence matrix at a threshold of 0.4, which resembles the ground truth shown in Fig. (d), although it is still not dense enough. Consequently, we computed the second-order similarity relationship $\tilde{\mathbf{K}}$ of $\mathbf{H}$, as depicted in Fig. (c), which more closely approximates the ground truth (We use Seeds dataset for this experiment)}.
    \label{fig1}
\end{figure*}

We extract the high-confidence elements in the CA matrix to approximate $\mathbf{K}^*$. This motivation is that if two samples belong to the same cluster, their pairwise value in CA matrix is more likely to be higher, which is reflected in the high-confidence elements of the CA\footnote[1]{\noindent In this context, the CA matrix does not solely represent the traditional CA matrix; any matrix that can express the similarity relationship between sample pairs is applicable, such as LWCA matrix \cite{huang2017locally}, NWCA matrix \cite{zhang2024similarity}, etc. In this paper, we employ NWCA matrix as $\bar{\mathbf{A}}$.} matrix, as illustrated in Fig. \ref{fig1}. It is evident that as the values in the CA matrix increase, the precision of the corresponding elements also improves. Therefore, high-confidence elements from the CA matrix can well approximate the ground-truth relationship between two samples. Specifically, the high-confidence elements are calculated by
\begin{align}\label{daitoue}
    \mathbf{H}_{ij}=\begin{cases}
	\mathbf{\bar{K}}_{ij},&		\mathbf{\bar{K}}_{ij}\ge \alpha ,\\
	0,&		\mathrm{else}\\
    \end{cases}.
\end{align}
where $\alpha$ is a hyper-parameter. Eq. (\ref{daitoue}) retains the high-confidence elements in the CA matrix and discards the low-confidence ones. However, $\mathbf{H}$ in Eq. (\ref{daitoue}) is generally very sparse (as illustrated in Fig. \ref{fig1}, the recall and proportion rates of the elements in $\mathbf{H}$ decrease as $\alpha$ increases) and not semi-positive definite. Therefore, we compute its second-order similarity relations to make it denser and semi-positive by 
\begin{equation}\label{Ktilde}
    \mathbf{\tilde{K}}=\left( \mathbf{D}^{-1} \right) ^{\top}\mathbf{H}^{\top}\mathbf{HD}^{-1},
\end{equation}
where $\mathbf{D}$ is a diagonal matrix and $\mathbf{D}_{ii}=\sqrt{\sum_{i=1}^n{\left( \mathbf{H}_{ij} \right) ^2}}$.
\subsection{Proposed Ensemble Clustering Algorithm}
After approximating $\mathbf{K}^*$ by $\tilde{\mathbf{K}}$ in Eq. (\ref{Ktilde}), Eq. (\ref{main}) can be instantiated into the following practical ensemble clustering method (we provide brief proof of Eq. (\ref{main}) $\Rightarrow$ Eq. (\ref{solve}) in Appendix \ref{proof_eq912}).
\begin{equation}
    \begin{aligned}\label{solve}
        \underset{\mathbf{w}}{\min}\,\,\underset{\mathbf{Z}}{\max}\,\,\text{tr}\left( \left(2 \mathbf{\tilde{K}}+\mathbf{K}^{\mathbf{w}} \right) \mathbf{ZZ}^{\top} \right) 
        \\
        \text{s}.\text{t}. \ \mathbf{Z}^{\top}\mathbf{Z}=\mathbf{I}, \sum_{t=1}^m {w_t}=1, w_t\ge 0.
    \end{aligned}
\end{equation}
Since both $\tilde{\mathbf{K}}$ and $\mathbf{K}^\mathbf{w}$ are positive semi-definite, the problem is a convex problem with respect to $\mathbf{w}$. Theoretically, we can obtain the global minimum point for $\mathbf{w}$. Once the spectral embedding $\mathbf{Z}$ is obtained, we apply $k$-means algorithm to it to derive the final discrete clustering results.
\begin{table*}[!ht]
    \centering
    \setlength{\tabcolsep}{3.8pt}
    \caption{Performance (\%) evaluation of different datasets based on the NMI metric. We have highlighted the values of the best-performing method in \textbf{bold}, and the second-best method is marked with an \underline{underline}.}
    \begin{tabular}{c|cccccccccc|c}
    \toprule
    Method & D1 & D2 & D3 & D4 & D5 & D6 & D7 & D8 & D9 & D10 & Average \\
    \midrule
    CEAM (TKDE'24)  & 5.6\tiny{$\pm$10} & 36.2\tiny{$\pm$26} & 16.7\tiny{$\pm$4} & 27.4\tiny{$\pm$1} & 60.1\tiny{$\pm$10} & 4.3\tiny{$\pm$3} & 18.0\tiny{$\pm$2} & 19.0\tiny{$\pm$5} & 14.3\tiny{$\pm$4} & 8.8\tiny{$\pm$5} & 21.0\tiny{$\pm$8}  \\
    CEs$^2$L (AIJ'19) & 3.4\tiny{$\pm$5} & 9.3\tiny{$\pm$10} & 19.0\tiny{$\pm$4} & 27.9\tiny{$\pm$2} & 45.1\tiny{$\pm$14} & 12.3\tiny{$\pm$5} & 12.0\tiny{$\pm$2} & 15.2\tiny{$\pm$7} & \underline{15.7}\tiny{$\pm$3} & 10.2\tiny{$\pm$6} & 17.0\tiny{$\pm$6}  \\
    CEs$^2$Q (AIJ'19) & 2.5\tiny{$\pm$4} & 11.5\tiny{$\pm$8} & 17.6\tiny{$\pm$5} & 28.1\tiny{$\pm$3} & 43.9\tiny{$\pm$15} & 12.1\tiny{$\pm$5} & 12.2\tiny{$\pm$2} & 17.9\tiny{$\pm$4} & 15.4\tiny{$\pm$3} & 7.5\tiny{$\pm$4} & 16.9\tiny{$\pm$6}  \\
    LWEA (TCYB'18) & 0.4\tiny{$\pm$0} & 53.3\tiny{$\pm$3} & 15.9\tiny{$\pm$3} & 28.1\tiny{$\pm$1} & 63.3\tiny{$\pm$3} & 12.1\tiny{$\pm$5} & 13.7\tiny{$\pm$3} & 21.0\tiny{$\pm$4} & 14.7\tiny{$\pm$1} & 7.9\tiny{$\pm$4} & 23.0\tiny{$\pm$3}  \\
    NWCA (arXiv'24)       & 0.4\tiny{$\pm$0} & 52.5\tiny{$\pm$3} & 16.0\tiny{$\pm$3} & 28.4\tiny{$\pm$1} & 63.7\tiny{$\pm$3} & 12.5\tiny{$\pm$4} & 13.6\tiny{$\pm$3} & 21.7\tiny{$\pm$1} & 14.8\tiny{$\pm$1} & 9.7\tiny{$\pm$4} & 23.3\tiny{$\pm$2}  \\
    ECCMS (TNNLS'24)  & 0.4\tiny{$\pm$0} & 50.7\tiny{$\pm$19} & 18.4\tiny{$\pm$5} & 28.2\tiny{$\pm$0} & 64.7\tiny{$\pm$3} & 12.3\tiny{$\pm$5} & 12.9\tiny{$\pm$3} & \underline{22.8}\tiny{$\pm$4} & 15.5\tiny{$\pm$2} & 9.1\tiny{$\pm$4} & 23.5\tiny{$\pm$5}  \\
    MKKM (arXiv'18)  & 8.1\tiny{$\pm$12} & 40.8\tiny{$\pm$20} & 12.8\tiny{$\pm$3} & 20.6\tiny{$\pm$6} & 55.4\tiny{$\pm$9} & 12.0\tiny{$\pm$5} & 19.7\tiny{$\pm$4} & 14.3\tiny{$\pm$4} & 12.0\tiny{$\pm$7} & 9.1\tiny{$\pm$6} & 20.5\tiny{$\pm$8}  \\
    SMKKM (TPAMI'23) & 8.7\tiny{$\pm$4} & 38.5\tiny{$\pm$11} & 19.3\tiny{$\pm$4} & 27.0\tiny{$\pm$2} & 59.4\tiny{$\pm$9} & 10.5\tiny{$\pm$5} & \underline{20.0}\tiny{$\pm$2} & 18.2\tiny{$\pm$3} & 15.5\tiny{$\pm$2} & 10.5\tiny{$\pm$4} & 22.8\tiny{$\pm$5}  \\
    SEC (TKDE'17) & 9.2\tiny{$\pm$12} & 24.9\tiny{$\pm$18} & 17.3\tiny{$\pm$4} & 21.9\tiny{$\pm$5} & 36.0\tiny{$\pm$17} & 12.8\tiny{$\pm$4} & 15.5\tiny{$\pm$3} & 13.6\tiny{$\pm$7} & 9.9\tiny{$\pm$6} & 7.1\tiny{$\pm$4} & 16.8\tiny{$\pm$9}  \\
    \midrule
    Proposed ($\alpha =0.1$) & \underline{25.0}\tiny{$\pm$12} & \underline{58.3}\tiny{$\pm$1} & \underline{20.0}\tiny{$\pm$4} & \underline{29.4}\tiny{$\pm$2} & \underline{67.5}\tiny{$\pm$3} & \underline{14.4}\tiny{$\pm$4} & 18.8\tiny{$\pm$2} & 19.6\tiny{$\pm$6} & 15.0\tiny{$\pm$4} & \underline{12.4}\tiny{$\pm$4} & \underline{28.0}\tiny{$\pm$4}  \\
    Proposed & \textbf{25.0}\tiny{$\pm$12} & \textbf{59.0}\tiny{$\pm$1} & \textbf{21.1}\tiny{$\pm$3} & \textbf{29.4}\tiny{$\pm$2} & \textbf{67.5}\tiny{$\pm$3} & \textbf{15.0}\tiny{$\pm$4} & \textbf{22.9}\tiny{$\pm$2} & \textbf{27.4}\tiny{$\pm$2} & \textbf{15.8}\tiny{$\pm$3} & \textbf{12.4}\tiny{$\pm$4} & \textbf{29.6}\tiny{$\pm$4}  \\
    \toprule
    \end{tabular}
    \label{NMI_performance}
\end{table*}

\begin{figure*}[ht]
    \centering
    \subfigure[MKKM]{
    \includegraphics[width=0.32\linewidth]{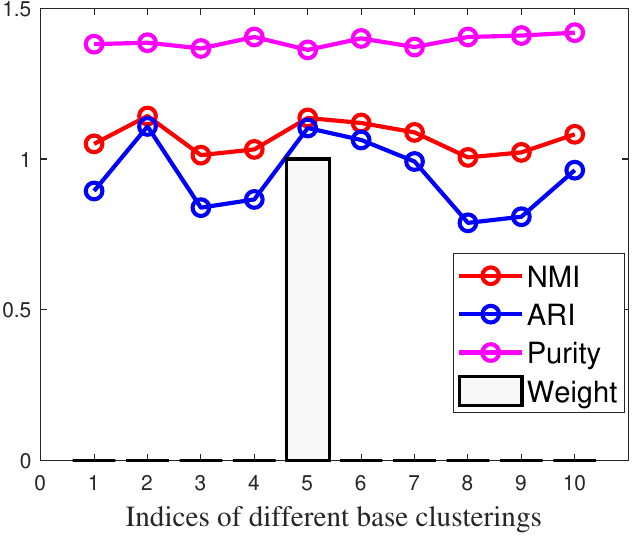}}
            \hspace{0.001\linewidth}
    \subfigure[SimpleMKKM]{
    \includegraphics[width=0.32\linewidth]{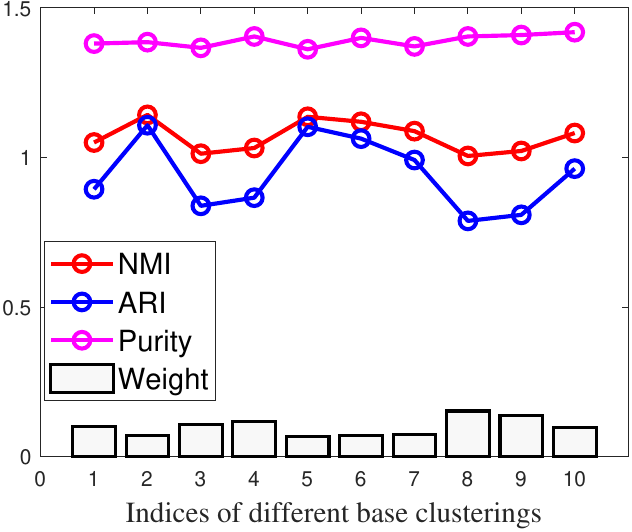}}
            \hspace{0.001\linewidth}
    \subfigure[The proposed method]{
    \includegraphics[width=0.32\linewidth]{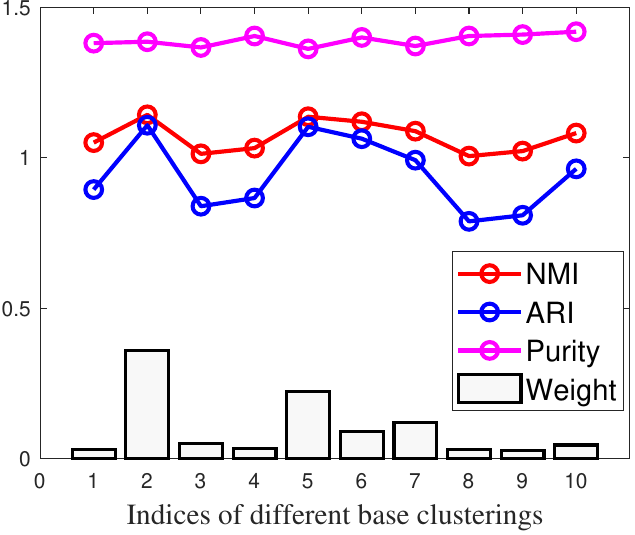}}
            \hspace{0.001\linewidth} 
    \caption{Illustration of the performance (line plot, NMI, ARI, and Purity) of each individual base clustering when the ensemble size is 10, as well as the clustering weights learned by the three different methods (bar plot). Note that all three metrics are better when they possess higher values. For better visualization, we add 0.5 to the values of each metric.}
    \label{compare_weight}
\end{figure*} 

\begin{table*}
    \caption{The detailed performance metrics of three different learning weight methods.}
    \centering
    \setlength{\tabcolsep}{5.5pt}
    \begin{threeparttable} 
        \begin{tabular}{ccccccccc}
            \toprule
            Method & Objective Function & Essence & NMI & ARI & Puritty & Bias$^*$ & $-$Diversity$^*$ & Total$^*$\\
            \midrule 
            MKKM & $\underset{\mathbf{w}}{\min}\,\,\underset{\mathbf{Z}}{\min}\,\,\text{tr}\left( \mathbf{K}^{\mathbf{w}}\left( \mathbf{I}-\mathbf{ZZ}^{\top} \right) \right) $ & min Diversity & 62.8 & 62.1 & 85.0 & \textbf{-94.6} & 116.2 & 21.6 \\
            SMKKM$^*$ & $\underset{\mathbf{w}}{\min}\,\,\underset{\mathbf{Z}}{\max}\,\,\text{tr}\left( \mathbf{K}^{\mathbf{w}}\mathbf{ZZ}^{\top} \right)$ & max Diversity & 65.3 & 66.2 & 85.9 & -10.0 & \textbf{1.0} & -9.0\\
            Proposed & \makecell{$\underset{\mathbf{Z}}{\max}\,\text{tr}\left( \mathbf{K}^*\mathbf{ZZ}^{\top} \right)$ \\ $\underset{\mathbf{w}}{\min}\,\underset{\mathbf{Z}}{\max}\,\text{tr}\left( \mathbf{K}^{\mathbf{w}}\mathbf{ZZ}^{\top} \right)$}  & \makecell{min Bias\\ max Diversity} & \textbf{71.9} & \textbf{72.4} & \textbf{90.2} & -24.0 & 6.4 & \textbf{-17.6} \\
            \bottomrule
        \end{tabular}
        \begin{tablenotes}
            \item \small{$^*$ “Bias” refers to $-2\mathrm{tr}(\mathbf{K}^\mathbf{w}\mathbf{K}^*)$, “$-$Diversity” is defined as $\mathrm{tr}(\mathbf{K}^\mathbf{w}\mathbf{K}^\mathbf{w})$ and “Total” is equal to “Bias $-$ Diversity”.}
        \end{tablenotes}
    \end{threeparttable}
    \label{detail_weight}
\end{table*}
\subsection{Optimization of Eq. (\ref{solve})} \label{optimization12}
Eq. (\ref{solve}) is a typical min-max optimization problem with multi-variables. The usual approach is to fix one variable and optimize the other, but this approach often fails to yield a globally optimal solution. In \cite{9857664}, the author transformed this problem into minimizing the optimal value function and employed reduced gradient descent for solving it. Given the similarity of our problem-solving approach to this method, we provide only the key steps of the optimization process. Readers are encouraged to refer to \cite{9857664} for more details.

For Eq. (\ref{solve}), we rewrite it as (the constraints have been omitted for brevity)
\begin{equation*}
    \underset{\mathbf{w}}{\min}\,\,\mathcal{J}\left( \mathbf{w} \right) , \mathcal{J}\left( \mathbf{w} \right) =\left\{ \underset{\mathbf{Z}}{\max}\,\,\text{tr}\left( \left( 2\tilde{\mathbf{K}}+\mathbf{K}^{\mathbf{w}} \right) \mathbf{ZZ}^{\top} \right) \right\} .
\end{equation*}

As established in \cite{doi:10.1137/S0036144596302644}, $\mathcal{J}(\mathbf{w})$ is differentiable,
\begin{equation*}
    \frac{\partial \mathcal{J}\left( \mathbf{w} \right)}{\partial w_t}=2w_t\text{tr}\left( \mathbf{K}^{\mathbf{w}}\mathbf{Z}^*\mathbf{Z}^{*\top} \right) ,
\end{equation*}
where $\mathbf{Z}^*=\{ \text{arg}\max _{\mathbf{Z}}\,\text{tr}( (2\tilde{\mathbf{K}}+\mathbf{K}^{\mathbf{w})}\mathbf{ZZ}^{\top} ) , \mathbf{Z}^{\top}\mathbf{Z}=\mathbf{I} \}$. Building upon this, we compute the gradient of $\mathcal{J}(\mathbf{w})$ as follows:
\begin{equation*}
    \left[ \nabla \mathcal{J}\left( \mathbf{w} \right) \right] _t=\frac{\partial \mathcal{J}\left( \mathbf{w} \right)}{\partial w_t}-\frac{\partial \mathcal{J}\left( \mathbf{w} \right)}{\partial w_u}\,\,\forall t\ne u, 
\end{equation*}
and
\begin{equation*}
    \left[ \nabla \mathcal{J}\left( \mathbf{w} \right) \right] _u=\sum_{t=1,t\ne u}^m{\frac{\partial \mathcal{J}\left( \mathbf{w} \right)}{\partial w_u}-\frac{\partial \mathcal{J}\left( \mathbf{w} \right)}{\partial w_t}},
\end{equation*}
where $w_u$ is not selected as the zero component of $\mathbf{w}$. To address the constraint $\mathbf{w}\ge 0$, the final descent direction is computed as
\begin{equation}\label{d_t}
    \begin{aligned}
        d_t=\begin{cases}
        	0,&		\text{if } w_t=0 \text{ and } \left[ \nabla \mathcal{J}\left( \mathbf{w} \right) \right] _t>0,\\
        	-\left[ \nabla \mathcal{J}\left( \mathbf{w} \right) \right] _t,&		\text{if } w_t>0 \text{ and } t\ne u,\\
        	-\left[ \nabla \mathcal{J}\left( \mathbf{w} \right) \right] _u,&		\text{if } t=u,\\
        \end{cases}
    \end{aligned}
\end{equation}

where $d_t$ is the $p$-th component of gradient vector $\mathbf{d}$. We use gradient descent to set $\mathbf{w}_{t+1}\leftarrow \mathbf{w}_t + \beta \mathbf{d}$ and continue until the algorithm converges, where $\beta$ is a learning rate.The pseudo code for this algorithm is provided in Appendix \ref{pseudo}.

%% file: Contents/Experiment.tex
\begin{figure*}[ht]
    \centering
    \subfigure[$m=\log\log n$ (diverge)]{
    \includegraphics[width=0.313\linewidth]{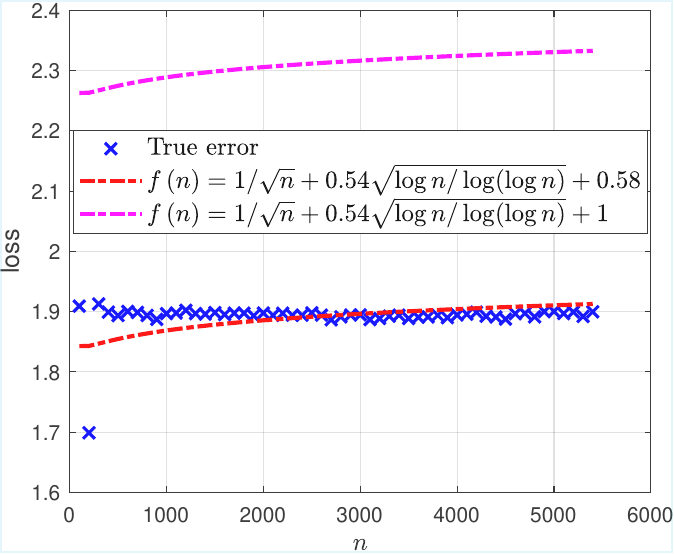}}
            \hspace{0.001\linewidth}
    \subfigure[$m=\log n$ (converge to $e,e>0$)]{
    \includegraphics[width=0.32\linewidth]{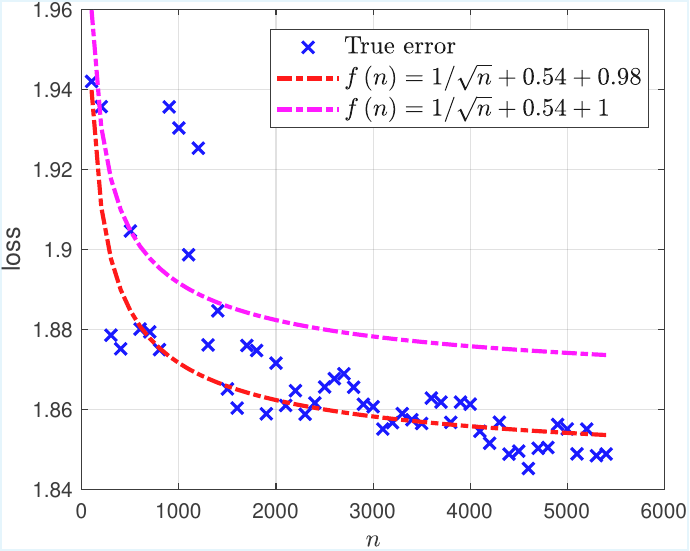}}
            \hspace{0.001\linewidth}
    \subfigure[$m=\sqrt{n}$ (converge to $0$)]{
    \includegraphics[width=0.32\linewidth]{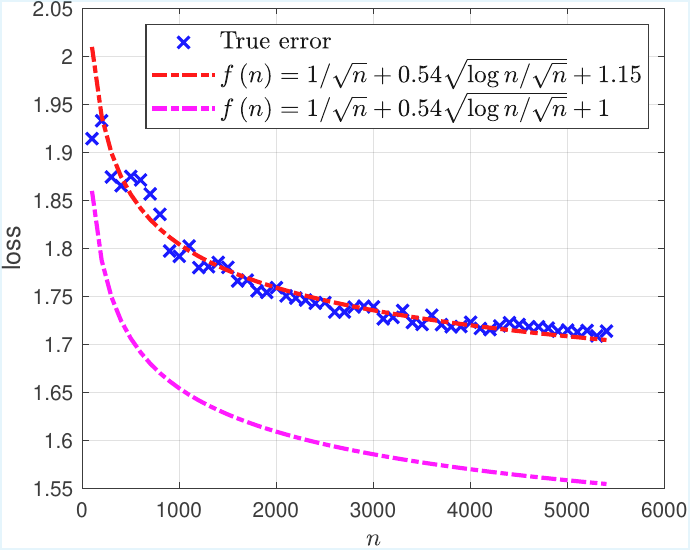}}
            \hspace{0.001\linewidth}
    \caption{We conducted experiments on real data for Theorem \ref{excess}. In this experiment, we uniformly sample data with an increment of 100 for the validation of $n$. The blue dots represent the errors computed from the real data, while the red and pink lines represent the fittings using the formulas from Theorem \ref{excess}.}
    \label{val_theorem}
\end{figure*}
\section{Experiments}\label{Experiment}
\subsection{Comparative Experiment}\label{ComExp}
We evaluated our method on 10 datasets with method CEAM \cite{10238807}, CEs$^2$L, CEs$^2$Q \cite{LI201937}, LWEA \cite{huang2017locally}, NWCA \cite{zhang2024similarity}, ECCMS \cite{jia2023ensemble}, MKKM \cite{bang2018robust}, SMKKM  \cite{9857664}, SEC \cite{7811216}.
Due to the space limitations, detailed descriptions of the datasets and comparison methods are provided in Appendix \ref{DeDa} and \ref{DeMe}. For each dataset, we repeat the experiments 20 times and compute the average performance. The true number of clustering class is chosen as $k$ for each dataset. Three performance metrics are selected to evaluate the methods: NMI, ARI, and Purity, and larger value indicates better performance. Table \ref{NMI_performance} reports the comparisons based on the NMI metric, while the results for ARI and Purity are provided in Appendix \ref{DeCom}. As shown in Table \ref{NMI_performance}, we observe that:
\begin{itemize}[noitemsep,topsep=0pt]
    \item The proposed method outperforms the comparison methods across all datasets. In terms of average performance, we exceed the second-best method by 6.1\%, 7.3\%, and 6.0\% in NMI, ARI, and Purity, respectively.
    \item On some difficulty datasets, the performance advantage of our method is more significant. For example, in the D1 (Phishing) dataset, the results obtained by other methods are close to 0 as measured by NMI, rendering them nearly impractical for guiding applications, while our method can provide some valuable information.  
    \item Even with the hyper-parameter $\alpha =0.1$ fixed, our method outperforms the methods compared in most datasets and lead on average across all methods. For example, with fixed hyper-parameter, we respectively lead the second-best method by 4.5\%, 6.2\%, and 4.4\% in NMI, ARI, and Purity on average.
\end{itemize}
To further substantiate the efficacy of the modified method, we carry out hyper-parameter sensitivity experiment, ablation study, and ensemble size experiment. Due to space limitations, these experiments are detailed in Appendices \ref{HyAn}, \ref{AbEx} and \ref{Size_Exp}.

\subsection{Clustering Weight Analysis}
In this section, we analyze the base clustering weights learned by our method and compare them with those of other base clusterings weighted averaging method: MKKM \cite{bang2018robust}, SimpleMKKM (SMKKM) \cite{9857664}. The details of these methods are summarized in Table \ref{detail_weight}, where we compare them across multiple metrics. As shown in Table \ref{detail_weight}, MKKM inherently reduces the diversity of base clusterings, thereby concentrating the weights on a single base clustering, as illustrated in Fig. \ref{compare_weight}. When the selected single base clustering aligns closely with the ground truth, MKKM significantly reduces bias, thereby enhancing clustering performance. However, this approach often faces two defects: 1) it does not always assign higher weight to the more accurate base clustering, as it lacks supervision; and 2) even if the best base clustering is selected every time, this method loses the advantage of ensemble learning that uses multiple base clusterings to achieve a better one. The objective function of the SimpleMKKM method is designed to enhance the diversity among base clusterings (as Table \ref{detail_weight} reports), thereby distributing the weights as possible across multiple distinct base clusterings, as illustrated in Fig. \ref{compare_weight}. However, their weight values seem to follow an opposite trend to the performance of individual clustering, validating our earlier discussion (\textbf{Remark 3} in Section \ref{Bias_Diversity_Decom}) that assigning weights solely to enhance diversity can lead to misallocation. Our method introduces high-confidence elements to guide the diversity and reduce bias, ensuring that higher weights are assigned to more accurate base clusterings. As a result, the proposed method achieves the lowest “Total” loss and enhances the final clustering performance. As a summary, the above analysis is consistent with our theoretical findings.

\subsection{Validation of Excess Risk Bound}
In this section, we validate Theorem \ref{excess} using the real dataset WFRN. Theorem \ref{excess} exhibits three distinct scenarios of excess risk: divergence, convergence to a constant greater than zero, and convergence to zero. We conduct experiments for these scenarios with $m = \log \log n$, $m = \log n$, and $m = \sqrt{n}$, respectively. Since we cannot obtain the expectation of the CA matrix, we substitute it with the similarity matrix produced by the ground-truth label. Therefore, the fitting function is defined as $a_1 \sqrt{n}+a_2\sqrt{\log n/m}+\text{gap}$, where the gap represents the difference between the similarity matrix of the labels and the expected value of the CA matrix. It can be observed in Fig. \ref{val_theorem} that the function fits the loss data points accurately when we choose $a_1=1$ and $a_2 = 0.54$, as indicated by the red line. Theoretically, our gap should be a fixed value, as illustrated by the pink line. Although it does not completely conform to our data, the trend is consistent with the loss. From this experiment, we observe that when $m=\log \log n$, the loss value remains almost stable at 1.9; when $m=\log n$, the loss value shows a clear decreasing trend and eventually arrives at 1.85; and when $m=\sqrt{n}$, the loss curve exhibits a steady decline, ultimately reaching a loss value of approximately 1.7. Combining this experiment with our theoretical analysis, we further demonstrate that when $m \ll \log n$, the excess risk diverges; when $m=\mathcal{O}(\log n)$, it converges to a constant greater than 0; and when $m\gg \log n$, it converges to 0.

%% file: Contents/Conclusion.tex
\section{Conclusion and Discussion}\label{Conclusion}
In this paper, we have presented the generalization error bound, excess risk bound, and sufficient conditions for the consistency of ensemble clustering. Through this, we have elucidated the interplay between sample size $n$ and the number of base clustering $m$, offering insights relevant to practical applications. By approximating clustering expectations using weighted finite clustering, we identified the impact of Bias and Diversity on the errors between them. Notably, we have shown that maximizing Diversity aligns closely with robust optimization principles. Our contribution extends to the introduction of a novel ensemble clustering algorithm rooted in our theoretical framework, which significantly outperforms other SOTA methods. 

It is also important to acknowledge certain limitations in our work. While we have established sufficient conditions for consistency, necessary conditions remain unaddressed, and the tightness of convergence rates for generalization error and excess risk is yet to be fully evidenced. The algorithm derived from our theory only represents a specific instance, leaving room for the exploration and comparison of diverse algorithms developed within this framework.

%% file: Contents/Appendics.tex
\begin{center}
    \Large\textbf{Appendix}
\end{center}
\section{Overview of the Appendix}
Our appendix consists of three main sections:
\begin{itemize}
    \item Proofs of the three theorems in Section \ref{generalization_per}, i.e., Theorem \ref{generalization} (generalization error bound), Theorem \ref{excess} (excess risk bound), and Theorem \ref{consistency} (sufficient conditions for consistency).
    \item Proof of Theorem \ref{BD_decom} in Section \ref{Bias_Diversity_Decom}, Eq. (\ref{buzhidaodingyishenmelabel}) $\Rightarrow$ Eq. (\ref{simplify_eq}) and Eq. (\ref{main}) $\Rightarrow$ Eq. (\ref{solve})
    \item Some information omitted from the main text due to the space limit, including pseudo code for Section \ref{optimization12}, related work, details of datasets and comparison methods, clustering performance on ARI and Purity metrics, as well as experiments on hyper-parameters, ablation studies, and ensemble size analysis.
\end{itemize}
To clarify our proof process, we provide sketches of the proofs for the theorems in Sections \ref{generalization_per} and \ref{Bias_Diversity_Decom} in Appendices \ref{A.1} and \ref{A.2}, respectively, and the detailed proofs are provided in Appendix \ref{detailProof}.
\subsection{The sketching proof of Theorem \ref{generalization}, \ref{excess}, \ref{consistency}}\label{A.1}
\subsubsection{The sketching proof of Theorem \ref{generalization}}\label{proof_gen}
To proof Theorem \ref{generalization}, we first make the following decomposition,
\begin{equation*}
        \hat{F}\left( \hat{Z};\bar{K} \right) -F\left( \mathcal{Z} ;K^* \right) =\underset{\mathcal{A}}{\underbrace{\hat{F}\left( \hat{Z};\bar{K} \right) -\hat{F}\left( \hat{\mathcal{Z}};K^* \right) }}+\underset{\mathcal{B}}{\underbrace{\hat{F}\left( \hat{\mathcal{Z}};K^* \right) -F\left( \mathcal{Z} ;K^* \right) }},
\end{equation*}
where $\hat{F}(\hat{\mathcal{Z}};K^*)$ is the empirical error with expected CA matrix $\mathbf{K}^*$ (function $K^*$),
\begin{equation*}
    \hat{F}( \hat{\mathcal{Z}};{K}^* ) =\underset{\left\{ \hat{\zeta}_q \right\} _{q=1}^{k}\in \Gamma}{\max}\frac{1}{n^2}\sum_{q=1}^k{\sum_{i=1}^n{\sum_{j=1}^n{{K}^*( x_i,x_j ) \hat{\zeta}_q( x_i ) \hat{\zeta}_q ( x_j )}}}.
\end{equation*}
$\mathcal{A}$ can be further decomposed as (note that $\hat{F}\left( \hat{{Z}};\bar{{K}} \right)$ is equivalent to $\hat{F}\left( \hat{\mathbf{Z}};\bar{\mathbf{K}} \right)$, $\hat{F}\left( \hat{{{\mathcal{Z}}}};{K}^* \right)$ is equivalent to $\hat{F}\left( \hat{{\boldsymbol{\mathcal{Z}}}};\mathbf{K}^* \right)$)
\begin{equation*}
    \begin{aligned}
        &\hat{F}\left( \hat{{Z}};\bar{{K}} \right) -\hat{F}\left( \hat{{{\mathcal{Z}}}};{K}^* \right)
        \\
        =&\frac{1}{n}\left( \mathrm{tr}\left( \hat{\mathbf{Z}}^{\top}\bar{\mathbf{K}}\hat{\mathbf{Z}} \right) -\mathrm{tr}\left( \hat{\boldsymbol{{\mathcal{Z}}}}^{\top}\mathbf{K}^*\hat{\boldsymbol{{\mathcal{Z}}}} \right) \right) 
        \\
        =&\underset{\mathcal{A} _1}{\underbrace{\frac{1}{n}\left( \mathrm{tr}\left( \hat{\mathbf{Z}}^{\top}\left( \bar{\mathbf{K}}-\mathbf{K}^* \right) \hat{\mathbf{Z}} \right) \right) }}+\underset{\mathcal{A} _2}{\underbrace{\frac{1}{n}\left( \mathrm{tr}\left( \hat{\mathbf{Z}}^{\top}\mathbf{K}^*\hat{\mathbf{Z}} \right) -\mathrm{tr}\left( \hat{\boldsymbol{{\mathcal{Z}}}}^{\top}\mathbf{K}^*\hat{\boldsymbol{{\mathcal{Z}}}} \right) \right) }}.
    \end{aligned}
\end{equation*}
Therefore, we prove Theorem \ref{generalization} by bounding $\mathcal{A}_1,\mathcal{A}_2,\mathcal{B}$ separately, which leads to the following three lemmas.
\begin{lemma}\label{A_1}
    Under the general assumptions, we have
    \begin{equation*}
        \mathcal{A}_1=\frac{1}{n}\left( \mathrm{tr}\left( \hat{\mathbf{Z}}^{\top}\left( \bar{\mathbf{K}}-\mathbf{K}^* \right) \hat{\mathbf{Z}} \right) \right) \le \frac{2}{3m}\log \frac{2n}{\delta}+\sqrt{\frac{8}{m}\log \frac{2n}{\delta}},
    \end{equation*}
    with probability at least $1-\delta$. (The detailed proof of Lemma \ref{A_1} is in Appendix \ref{B.1})
\end{lemma}

\begin{lemma}\label{A_2}
    Under the general assumptions and assume that the gap between the $k$-th and $(k + 1)$-th eigenvalues of the expectation of normalized similarity matrix $\mathbf{K}^*$ is $\delta_k$ and $\delta_k\ge \frac{1}{c}>0$ where $c$ is a constant, we have
    \begin{equation*}
        \mathcal{A}_2=\frac{1}{n}\left( \mathrm{tr}\left( \hat{\mathbf{Z}}^{\top}\mathbf{K}^*\hat{\mathbf{Z}} \right) -\mathrm{tr}\left( \hat{\boldsymbol{{\mathcal{Z}}}}^{\top}\mathbf{K}^*\hat{\boldsymbol{{\mathcal{Z}}}} \right) \right) \le 2\sqrt{2}c\left(\frac{2}{3m}\log \frac{2n}{\delta}+\sqrt{\frac{8}{m}\log \frac{2n}{\delta}}\right),
    \end{equation*}
    with probability at least $1-\delta$. (The detailed proof of Lemma \ref{A_2} is in Appendix \ref{B.2})
\end{lemma}

\begin{lemma}\label{A_3}
    Under the general assumptions, we have
    \begin{equation*}
        \mathcal{B}=\hat{F}\left( \hat{\mathcal{Z}};K^* \right) -F\left( \mathcal{Z} ;K^* \right)\le \frac{2\sqrt{2}\log (\frac{2}{\delta})}{\sqrt{n}},
    \end{equation*}
    with probability at least $1-\delta$. (The detailed proof of Lemma \ref{A_3} is in Appendix \ref{B.3})
\end{lemma}
For Lemma \ref{A_1}, our proof primarily relies on matrix Bernstein inequality. In the case of Lemma \ref{A_2}, we apply perturbation theory to derive the bound for $\mathcal{A}_2$. The proof of Lemma \ref{A_3} is mainly concerned with the integral operator theory of \cite{JMLR:v11:rosasco10a}. By combining Lemmas \ref{A_1}, \ref{A_2}, and \ref{A_3}, we have
\begin{align*}
    F\left( \hat{Z};K^* \right) -F\left( \mathcal{Z};K^* \right) 
    \le \left(2\sqrt{2}c+1\right)\left(\frac{2}{3m}\log \frac{6n}{\delta}+\sqrt{\frac{8}{m}\log \frac{6n}{\delta}}\right)+\frac{2\sqrt{2}\log \left(\frac{6}{\delta}\right)}{\sqrt{n}} 
\end{align*}
with at least probability $1-\delta$, which completes the proof of Theorem \ref{generalization}. \qed

\vspace{0.2cm}
\subsubsection{The sketching proof of Theorem \ref{excess}}\label{proof_excess}
Based on the proof of Theorem \ref{generalization}, we have
\begin{equation*}
    \begin{aligned}
        &F\left( \hat{Z};K^* \right) -F\left( \mathcal{Z} ;K^* \right) 
        \\
        =&F\left( \hat{Z};K^* \right) -\hat{F}\left( \hat{Z};\bar{K} \right) +\underset{\text{Generalization error}}{\underbrace{\hat{F}\left( \hat{Z};\bar{K} \right) -F\left( \mathcal{Z} ;K^* \right) }}
        \\
        =&F\left( \hat{Z};K^* \right) -\hat{F}\left( \hat{Z};K^* \right) +\hat{F}\left( \hat{Z};K^* \right) -\hat{F}\left( \hat{Z};\bar{K} \right) +\underset{\mathcal{A}}{\underbrace{\hat{F}\left( \hat{Z};\bar{K} \right) -\hat{F}\left( \hat{\mathcal{Z}};K^* \right) }}+\underset{\mathcal{B}}{\underbrace{\hat{F}\left( \hat{\mathcal{Z}};K^* \right) -F\left( \mathcal{Z} ;K^* \right) }}
        \\
        =&\underset{\mathcal{C}}{\underbrace{F\left( \hat{Z};K^* \right) -\hat{F}\left( \hat{Z};K^* \right) }}+\underset{-\mathcal{A} _1}{\underbrace{\hat{F}\left( \hat{Z};K^* \right) -\hat{F}\left( \hat{Z};\bar{K} \right) }}+\underset{\mathcal{A} _1}{\underbrace{\hat{F}\left( \hat{Z};\bar{K} \right) -\hat{F}\left( \hat{Z};K^* \right) }}
        \\
        +&\underset{\mathcal{A} _2}{\underbrace{\hat{F}\left( \hat{Z};K^* \right) -\hat{F}\left( \hat{\mathcal{Z}};K^* \right) }}+\underset{\mathcal{B}}{\underbrace{\hat{F}\left( \hat{\mathcal{Z}};K^* \right) -F\left( \mathcal{Z} ;K^* \right) }}
        \\
        =&\underset{\mathcal{C}}{\underbrace{F\left( \hat{Z};K^* \right) -\hat{F}\left( \hat{Z};K^* \right) }}+\underset{\mathcal{A} _2}{\underbrace{\hat{F}\left( \hat{Z};K^* \right) -\hat{F}\left( \hat{\mathcal{Z}};K^* \right) }}+\underset{\mathcal{B}}{\underbrace{\hat{F}\left( \hat{\mathcal{Z}};K^* \right) -F\left( \mathcal{Z} ;K^* \right) }}.
    \end{aligned}
\end{equation*}
Therefore, we only need to bound $\mathcal{C}$, as the bounds of $\mathcal{A}_2$ and $\mathcal{B}$ can be obtained directly from Lemmas \ref{A_2}, \ref{A_3}.
\begin{lemma}\label{A_4}
    Under the general assumptions and with the additional condition that $||\hat{z}_q||_{\infty} \le c_0$ ($c_0 > 0$ is a constant), we have
    \begin{equation*}
        \mathcal{C} = F(\hat{Z};K^*)-\hat{F}(\hat{Z};K^*) \le k\left( \frac{2\sqrt{2}c_0}{\sqrt{n}}+\sqrt{\frac{8\log \frac{1}{\delta}}{n}} \right),
    \end{equation*}
    with probability at least $1-\delta$. (The detailed proof of Lemma \ref{A_4} is in Appendix \ref{B.4})
\end{lemma}
For bounding $\mathcal{C}$, we utilize the McDiarmid's inequality \cite{McDiarmid_1989} and Rademacher complexity. The former is a standard tool to bound the difference of the random variable and its expectation. The reason for utilizing the latter technology is that, $\hat{F}(\hat{Z};K^*)$ is a pairwise function and some tools in the i.i.d. condition is not satisfied \cite{Li_Ouyang_Liu_2023}. We derive the bound of $\mathcal{C}$ analogously to \cite{Li_Ouyang_Liu_2023}. By combining Lemmas \ref{A_2}, \ref{A_3} and \ref{A_4}, we derive that
\begin{equation*}
    F(\hat{Z};K^*)-F(\mathcal{Z};K^*) \le k\left(\frac{2\sqrt{2}c_0}{\sqrt{n}}+\sqrt{\frac{8\log \frac{3}{\delta}}{n}}\right)+2\sqrt{2}c\left(\frac{2}{3m}\log \frac{6n}{\delta}+\sqrt{\frac{8}{m}\log \frac{6n}{\delta}}\right)+\frac{2\sqrt{2}\log \left( \frac{6}{\delta} \right)}{\sqrt{n}}.
\end{equation*}
with probability at least $1-\delta$. This concludes the proof of Theorem \ref{excess}. \qed

\subsubsection{The sketching proof of Theorem \ref{consistency}}\label{proof_consistency}
Theorem \ref{consistency} describes that empirical eigenvectors ($\hat{\mathbf{z}}_q$) of CA matrix ($\bar{\mathbf{K}}$) converge to the eigenfunctions ($\zeta_q$) of integral operator ($L_{K^*}$) in probability in the limit case. We introduce the intermediate vector $\hat{\boldsymbol{z}}_q$ ($\hat{\boldsymbol{z}}_q$ is the eigenvector of expected CA matrix $\mathbf{K}^*$) and proceed with the following decomposition:
\begin{equation*}
    \lVert a_q\mathbf{\hat{z}}_{q}-\zeta_q \rVert _{\infty}
    \le \underset{\mathcal{M}}{\underbrace{\lVert a_q\mathbf{\hat{z}}_{q}-b_q\boldsymbol{\hat{z}}_{q} \rVert _{\infty}}}+\underset{\mathcal{N}}{\underbrace{\lVert b_q\boldsymbol{\hat{z}}_{q}-\zeta _q \rVert _{\infty}}}.
\end{equation*}
For $\mathcal{N}$, we know that there exist a sequence $(b_q)_q\in \{-1,1\}$ such that $\lVert b_q\boldsymbol{\hat{z}}_{q}-\zeta _q \rVert _{\infty}\to 0$ as $n \to \infty$, which has been proved in the Theorem 15 by \cite{von2008consistency}. We need only prove that $\mathcal{M}$ converges to 0 in probability.
\begin{lemma}\label{A_5}
    Under the same assumptions as Lemma \ref{A_2}, there exists a sequence $(a_q)_q \in \{-1,1\}$ such that
    \begin{equation*}
        \lVert a_q\mathbf{\hat{z}}_{q}-b_q\boldsymbol{\hat{z}}_{q} \rVert _{\infty} \to 0,
    \end{equation*}
    in probability as $m,n \to \infty$ and $m \gg \log n$. (The detailed proof of Lemma \ref{A_5} is in Appendix \ref{B.5})
\end{lemma}
For Lemma \ref{A_5}, our proof technique primarily relies on perturbation theory and trigonometric functions transformations. By incorporating $\mathcal{M}$, we complete the proof of Theorem \ref{consistency}. \qed

\subsection{The sketching proof of Theorem \ref{BD_decom} and Eq. (\ref{buzhidaodingyishenmelabel}) $\Rightarrow$ Eq. (\ref{simplify_eq})}\label{A.2}
\subsubsection{The sketching proof of Theorem \ref{BD_decom}}\label{proof_theorem41}
Theorem \ref{BD_decom} presents the bias-diversity decomposition for ensemble clustering. To prove this theorem, we introduce the following two lemmas.
\begin{lemma}\label{A_6}
    According to the definitions in Section \ref{Bias_Diversity_Decom}, where $\mathbf{K}^{\mathbf{w}}=\sum_{t=1}^m w_t\mathbf{K}^{(i)}$, $\mathbf{K}^*$ is the expectation of $\mathbf{K}^{(i)}$, and $w_t$ is the weight of $t$-th base clusterings. We have the following decomposition
    \begin{equation*}
        \lVert \mathbf{K}^{\mathbf{w}}-\mathbf{K}^* \rVert _{\mathrm{F}}^{2}=\frac{1}{m}\sum_{t=1}^m{\lVert mw_t\mathbf{K}^{\left( t \right)}-\mathbf{K}^* \rVert _{\mathrm{F}}^{2}}-\frac{1}{m}\sum_{t=1}^m{\lVert mw_t\mathbf{K}^{\left( t \right)}-\mathbf{K}^{\mathbf{w}} \rVert _{\mathrm{F}}^{2}}.
    \end{equation*}
    The detailed proof of Lemma \ref{A_6} is in Appendix \ref{B.6}.
\end{lemma}
\begin{lemma}\label{A_7}
    Under the same assumptions as Lemma \ref{A_2}, we derive that
    \begin{equation*}
    \hat{F}\left( \mathbf{\hat{Z}};\mathbf{K}^{\mathbf{w}} \right) -\hat{F}\left( \hat{\boldsymbol{\mathcal{Z}}};\mathbf{K}^* \right) 
    \le \left( \frac{k}{n}+\frac{2\sqrt{2}}{\lambda _k\left( \mathbf{K}^{*} \right) -\lambda _{k+1}\left( \mathbf{K}^{*} \right)} \right) \lVert \mathbf{K}^{\mathbf{w}}-\mathbf{K}^* \rVert _{\mathrm{F}}.
    \end{equation*}
    The detailed proof of Lemma \ref{A_7} is in Appendix \ref{B.7}.
\end{lemma}
The proof of Lemma \ref{A_6} primarily relies on the properties of the matrix trace. The proof of Lemma \ref{A_7} employs tools similar to those used in Lemma \ref{A_2}. By combining these two lemmas and setting $c'=k/n+{2\sqrt{2}}/(\lambda_k(\mathbf{K}^*)-\lambda_{k+1})$, we can readily prove Theorem \ref{BD_decom}. \qed

\subsubsection{The sketching proof of Eq. (\ref{buzhidaodingyishenmelabel}) $\Rightarrow$ Eq. (\ref{simplify_eq})}\label{proof_eq78}
It can be observed that the coefficient $c'$ in Eq. (\ref{buzhidaodingyishenmelabel}) is a constant greater than zero (given the number of samples $n$, the number of base clusterings $m$, and the number of clusters $k$). Therefore, Eq. (\ref{buzhidaodingyishenmelabel}) is entirely equivalent to
\begin{equation*}
    \underset{\mathbf{w}}{\min}\,\sum_{t=1}^m|| \tilde{w}_t\mathbf{K}^{( t )}-\mathbf{K}^* || _{\text{F}}^{2}-\sum_{t=1}^m{\lVert \tilde{w}_t\mathbf{K}^{( t )}-\mathbf{K}^{\mathbf{w}} \rVert _{\text{F}}^{2}}.
\end{equation*}
Through equivalent transformation, we arrive at the following lemma.
\begin{lemma}\label{A_8}
    With the same definition of Lemma \ref{A_6}, we have
    \begin{equation*}
    \begin{aligned}
        \underset{\mathbf{w}}{\min}\,\sum_{t=1}^m|| \tilde{w}_t\mathbf{K}^{( t )}-\mathbf{K}^* || _{\text{F}}^{2}-\sum_{t=1}^m{\lVert \tilde{w}_t\mathbf{K}^{( t )}-\mathbf{K}^{\mathbf{w}} \rVert _{\mathrm{F}}^{2}} \ \Leftrightarrow \ \underset{\mathbf{w}}{\min}\,\,-2\mathrm{tr}\left( \mathbf{K}^{\mathbf{w}}\mathbf{K}^* \right) +\mathrm{tr}\left( \mathbf{K}^{\mathbf{w}}\mathbf{K}^{\mathbf{w}} \right).
    \end{aligned}
    \end{equation*}
    The detailed proof of Lemma \ref{A_8} is in Appendix \ref{B.8}.
\end{lemma}
Through Lemma \ref{A_8}, we can easily derive Eq. (\ref{simplify_eq}) from Eq. (\ref{buzhidaodingyishenmelabel}) . \qed

\subsubsection{The sketching proof of Eq. (\ref{main}) $\Rightarrow$ Eq. (\ref{solve})}\label{proof_eq912}
Eq. (\ref{main}) is defined as 
\begin{align*}
    \underset{\mathbf{Z}\in \mathbb{R} ^{n\times k}}{\max}\,\,&2\mathrm{tr}\left( \mathbf{K}^*\mathbf{ZZ}^{\top} \right) +\underset{\mathbf{w}}{\min}\max_{\mathbf{Z}\in \mathbb{R} ^{n\times k}} \,\,\mathrm{tr}\left( \mathbf{K}^{\mathbf{w}}\mathbf{ZZ}^{\top} \right) 
    \\
    &\mathrm{s}.\mathrm{t}. \mathbf{Z}^{\top}\mathbf{Z}=\mathbf{I},\mathbf{w}^{\top}1=1,\mathbf{w}\ge 0.
\end{align*}
In this optimization problem, the first term does not contain the optimization variable $\mathbf{w}$, so we can directly combine it with the second term to obtain (we omit the constraints for the sake of brevity)
\begin{align*}
    \underset{\mathbf{w}}{\min}\max_{\mathbf{Z}\in \mathbb{R} ^{n\times k}} \,\,\left( \mathrm{tr}\left( \mathbf{K}^{\mathbf{w}}\mathbf{ZZ}^{\top} \right) +2\mathrm{tr}\left( \mathbf{K}^*\mathbf{ZZ}^{\top} \right) \right) .
\end{align*}
Based on the properties of the matrix trace, we can derive that
\begin{align*}
    \underset{\mathbf{w}}{\min}\max_{\mathbf{Z}\in \mathbb{R} ^{n\times k}} \,\mathrm{tr}\left( \left( 2\mathbf{K}^*+\mathbf{K}^{\mathbf{w}} \right) \mathbf{ZZ}^{\top} \right) .
\end{align*}
By replacing $\mathbf{K}^*$ with $\tilde{\mathbf{K}}$ in Eq. (\ref{Ktilde}), we finally obtain Eq. (\ref{solve}).

\section{Detailed Proof}\label{detailProof}
In this section, we provide detailed proofs for each lemma presented in Appendices  \ref{A.1} and \ref{A.2}.

\subsection{Proof of Lemma \ref{A_1}}\label{B.1}
To prove Lemma \ref{A_1}, we need to introduce the following matrix Bernstein inequality \cite{Vershynin_2018}.
\begin{lemma}\label{Bernstein}
(Matrix Bernstein Inequality)
    Let $\mathbf{X}^{(1)},\cdots,\mathbf{X}^{(m)}$ be $0$-mean $n\times n$ symmetric independent matrices such that $\lVert \mathbf{X}^{(t)}\rVert\le C$ ($C$ is a constant) almost surely for all $t$. Then, $\forall\  \varepsilon > 0$, we have
    \begin{equation*}
        P\left( \lVert \frac{1}{m}\sum_{t=1}^m{\mathbf{X}^{\left( t \right)}} \rVert _{2}\ge \varepsilon \right) \le 2n\exp \left\{ -\frac{m^2\varepsilon ^2}{2\left( \sigma ^2+\frac{m\varepsilon C}{3} \right)} \right\} ,
    \end{equation*}
    where $\sigma ^2=\lVert \sum_{i=1}^m{\mathbb{E}}\left[ \mathbf{X}^{(t)2} \right] \rVert _{2}$.
\end{lemma}
\textit{Proof}. For $\mathcal{A}_1$, we have
\begin{equation*}
    \begin{aligned}
        \mathcal{A} _1&=\frac{1}{n}\left( \mathrm{tr}\left( \hat{\mathbf{Z}}^{\top}\left( \bar{\mathbf{K}}-\mathbf{K}^* \right) \hat{\mathbf{Z}} \right) \right) 
        \\
        &\le \frac{n}{n}\left\| \hat{\mathbf{Z}}^{\top}\left( \bar{\mathbf{K}}-\mathbf{K}^* \right) \hat{\mathbf{Z}}^{\top} \right\| _{2}
        \\
        &\le \left\| \hat{\mathbf{Z}} \right\| _{2}^{2}\left\| \bar{\mathbf{K}}-\mathbf{K}^* \right\| _{2}=\left\| \bar{\mathbf{K}}-\mathbf{K}^* \right\| _{2}
    \end{aligned}
\end{equation*}
Define $\mathbf{X}^{(t)}=\mathbf{K}^{(t)}-\mathbf{K}^*$, obviously we have $\mathbb{E}[\mathbf{X}^{(t)}]=\mathbb{E}[\mathbf{K}^{(t)}]-\mathbf{K}^*=0$. For $\sigma^2$, we have

\begin{equation*}
    \begin{aligned}
    \sigma ^2&=\left\| \sum_{i=1}^m{\mathbb{E} \left[ \mathbf{X}^{\left( t \right) 2} \right]} \right\| _{2}
    \\
    &=\left\| \sum_{i=1}^m{\mathbb{E} \left[ \left( \mathbf{K}^{\left( t \right)}-\mathbf{K}^* \right) ^2 \right]} \right\| _{2}
    \\
    &=\left\| \sum_{i=1}^m{\left( \mathbb{E} \left[ \mathbf{K}^{\left( t \right) 2} \right] -\mathbb{E} \left[ \mathbf{K}^{\left( t \right)}\mathbf{K}^* \right] -\mathbb{E} \left[ \mathbf{K}^*\mathbf{K}^{\left( t \right)} \right] +\mathbb{E} \left[ \mathbf{K}^{*2} \right] \right)} \right\| _{2}
    \\
    &=\left\| \sum_{i=1}^m{\mathbb{E} \left[ \mathbf{K}^{\left( t \right) 2} \right] -m\mathbf{K}^{*2}} \right\| _{2} \le m\,\,\mathop {\mathrm{sup}}_t\left\| \mathbb{E} \left[ \mathbf{K}^{\left( t \right) 2} \right] -\mathbf{K}^{*2} \right\| _{2}
    \\
    & \le m \sup_t \left\| \mathbb{E}[\mathbf{K}^{(t)2}] \right\|_2 + \left\| \mathbf{K}^{*2} \right\|_2
    \end{aligned}
\end{equation*}

Based on Jensen's inequality and $\mathbf{K}^{(t)} \preceq \mathbf{I}$, $\mathbf{K}^* \preceq \mathbf{I}$, we have
\begin{equation*}
    \left\|\mathbb{E}\left[\mathbf{K}^{(k)2}\right] \right\|_2 \le \mathbb{E}\left\| \mathbf{K}^{(t)2} \right\|_2 \le \mathbb{E} \left\| \mathbf{K}^{(t)} \right\|_2^2 \le \left\| \mathbf{I} \right\|_2^2, \quad \left\| \mathbf{K}^{*2} \right\|_2 \le \left\| \mathbf{K}^{*} \right\|_2^2\le \left\| \mathbf{I} \right\|_2^2,
\end{equation*}
for any $t$. Therefore, we can bound $\sigma^2$ by
\begin{equation*}
    \sigma ^2\le m\,\,\underset{t}{\mathrm{sup}}\left\| \mathbb{E} \left[ \mathbf{K}^{\left( t \right) 2} \right] -\mathbf{K}^{*2} \right\| \le m \left\| \mathbb{E} \left[\mathbf{K}^{(t)2}\right] \right\|_2+m\left\| \mathbf{K}^{*2} \right\|_2 \le 2m\left\| \mathbf{I} \right\|_2^2 = 2m.
\end{equation*}

With \textbf{Lemma} \ref{Bernstein}, we have
\begin{equation*}
    \begin{aligned}
        \mathcal{A}_1\le \frac{2}{3m}\log \frac{2n}{\delta}+\sqrt{\frac{8}{m}\log \frac{2n}{\delta}},
    \end{aligned}
\end{equation*}
with probability at least $1-\delta$. \qed

\subsection{Proof of Lemma \ref{A_2}}\label{B.2}
We use a variant of Davis-Kahan theory \cite{10.1093/biomet/asv008} to bound $\mathcal{A}_2$.
\begin{lemma}\label{A2}
    (Davis-Kahan theory) Assume $\mathbf{X}$ and $\mathbf{X}'$ are two $n\times n$ real symmetric matrices and their largest $d$ eigenvalues are $\lambda_1\ge \lambda_2\ge\cdots \ge \lambda_d$ and $\lambda'_1\ge \lambda'_2\ge\cdots \ge \lambda'_d$, the matrices $\mathbf{Z}$ and $\mathbf{Z}'$ are composed of corresponding eigenvectors $\mathbf{Z}=[\mathbf{z}_1,\cdots,\mathbf{z}_d]$ and $\mathbf{Z}'=[\mathbf{z}_1',\cdots,\mathbf{z}_d']$, we have
    \begin{equation*}
        \left\| \sin \Theta \right\| _{\mathrm{F}}\le \frac{2\min \left( d^{1/2}\left\| \mathbf{X}-\mathbf{X}^{\prime} \right\| _2,\left\| \mathbf{X}-\mathbf{X}^{\prime} \right\| _{\mathrm{F}} \right)}{  \lambda _d-\lambda _{d+1} },
    \end{equation*}
    where $\Theta=(\theta_1=\cos^{-1}\sigma_1,\cdots,\theta_d=\cos^{-1}\sigma_d)^\top$, $\theta_1,\cdots,\theta_d$ are the singular values of $\mathbf{Z}^\top \mathbf{Z}'$, $\sin (\Theta)$ is the $d \times d$ diagonal matrix with the elements $\sin (\theta)_{ii}=\sin (\theta_i)$. 
\end{lemma}
\textit{Proof}. For $\mathcal{A}_2$, we have
\begin{equation*}
    \begin{aligned}
        \mathcal{A} _2&=\frac{1}{n}\left( \mathrm{tr}\left( \hat{\mathbf{Z}}^{\top}\mathbf{K}^*\hat{\mathbf{Z}} \right) -\mathrm{tr}\left( \hat{\boldsymbol{\mathcal{Z}}}^{\top}\mathbf{K}^*\hat{\boldsymbol{\mathcal{Z}}} \right) \right) 
        \\
        &=\frac{1}{n}\left\| \mathbf{K}^* \right\| _{\mathrm{F}}\left\| \hat{\mathbf{Z}}\hat{\mathbf{Z}}^{\top}-\hat{\boldsymbol{\mathcal{Z}}}\hat{\boldsymbol{\mathcal{Z}}}^{\top} \right\| _{\mathrm{F}}
        \\
        &\le \left\| \mathbf{K}^* \right\| _2\left\| \hat{\mathbf{Z}}\hat{\mathbf{Z}}^{\top}-\hat{\boldsymbol{\mathcal{Z}}}\hat{\boldsymbol{\mathcal{Z}}}^{\top} \right\| _{\mathrm{F}}
        \\
        &=\left\| \hat{\mathbf{Z}}\hat{\mathbf{Z}}^{\top}-\hat{\boldsymbol{\mathcal{Z}}}\hat{\boldsymbol{\mathcal{Z}}}^{\top} \right\| _{\mathrm{F}}
        \\
        &=\sqrt{2}\left\| \sin \left( \Theta (\hat{\mathbf{Z}},\hat{\boldsymbol{\mathcal{Z}}}) \right) \right\| _{\mathrm{F}}
        \\
        &\le \frac{2\sqrt{2}\left\| \bar{\mathbf{K}}-\mathbf{K}^* \right\| _2}{\lambda _k\left( {\mathbf{K}^*} \right) -\lambda _{k+1}\left( {\mathbf{K}^*} \right)}
        \\
        &\le 2\sqrt{2}c \left( \frac{2}{3m}\log \frac{2n}{\delta}+\sqrt{\frac{8}{m}\log \frac{2n}{\delta}}\right),
    \end{aligned}
\end{equation*}
with probability at least $1-\delta$, where $c=\lambda_k({\mathbf{K}}^*)-\lambda_{k+1}({\mathbf{K}}^*)$. \qed

\subsection{Proof of Lemma \ref{A_3}}\label{B.3}
We introduce two integral operator in \cite{JMLR:v11:rosasco10a} to prove Lemma \ref{A_3}.

Assume $\mathcal{H}$ is the RKHS associate with kernel function $K(x,y)$, the empirical covariance operator $T_n:\mathcal{H} \rightarrow \mathcal{H} $ is defined as
\begin{equation*}   
    T_n=\frac{1}{n}\sum_{i=1}^n{\left< \cdot ,K_{x_i} \right> K_{x_i}},
\end{equation*}
where $K_{x_i}=K(x_i,\cdot)$. The expected covariance operator $T_\mathcal{H}:\mathcal{H}\to \mathcal{H}$ is
\begin{equation*}
    T_{\mathcal{H}}=\int_{\mathcal{X}}{\left< K_x,\cdot \right> K_x}\mathrm{d}\rho \left( x \right).
\end{equation*}

\textit{Proof}. By the definition of $\hat{F}\left( \hat{\mathcal{Z}};K^* \right) $ and $F\left( \mathcal{Z} ;K^* \right)$, we have
\begin{equation*}
    \begin{aligned}
        \hat{F}\left( \hat{\mathcal{Z}};K^* \right) -F\left( \mathcal{Z} ;K^* \right) &=\frac{1}{n^2}\sum_{q=1}^k{\sum_{i=1}^n{\sum_{j=1}^n{K^*\left( x_i,x_j \right) \hat{\zeta}_q\left( x_i \right) \hat{\zeta}_q\left( x_j \right)}}}-\sum_{q=1}^k{\iint_{\mathcal{X}}{K^*\left( x,y \right) \zeta _q\left( x \right) \zeta _q\left( y \right)}\mathrm{d}\rho \left( x \right) \mathrm{d}\rho \left( y \right)}
        \\
        &=\frac{1}{n}\sum_{q=1}^k{\sum_{i=1}^n{\hat{\ell}_q\hat{\zeta}_q\left( x_i \right) \hat{\zeta}_q\left( x_i \right)}}-\sum_{q=1}^k{\int_{\mathcal{X}}{\ell _q\zeta _q\left( x \right) \zeta _q\left( x \right)}\mathrm{d}\rho \left( x \right)}
        \\
        &=\sum_{q=1}^k{\hat{\ell}_q\hat{\boldsymbol{\zeta}}_{q}^{\top}\hat{\boldsymbol{\zeta}}_q}-\sum_{q=1}^k{\ell _q\int_{\mathcal{X}}{\zeta _q\left( x \right) \zeta _q\left( x \right)}\mathrm{d}\rho \left( x \right)}
        \\
        &=\sum_{q=1}^k{\left( \hat{\ell}_q-\ell _q \right)}
    \end{aligned}
\end{equation*}
where $\{\hat{\ell}_q\}_{q=1}^k,\{\ell_q\}_{q=1}^k$ are the largest $k$ eigenvalues of integral operators $\hat{L}_{K^*} \hat{\zeta}_q(x), L_{K^*}\zeta_q(x)$, respectively.

According to \cite{JMLR:v11:rosasco10a}, the eigenvalues of $\hat{L}_K^* $
and $T_n$ (with kernel function $K^*$) are the same up to 0, so do $L_{K^*}$ and $T_\mathcal{H}$ (with kernel function $K^*$). Therefore, we have
\begin{equation*}
    \begin{aligned}
    \hat{F}\left( \hat{\mathcal{Z}};K^* \right) -F\left( \mathcal{Z} ;K^* \right) =\sum_{q=1}^k{\left( \hat{\ell}_q-\ell _q \right)} \le \left| \sum_{q=1}^k{\hat{\sigma}_q-\sigma _q} \right| \le \left| \mathrm{tr}\left( T_n \right) -\mathrm{tr}\left( T_{\mathcal{H}} \right) \right| 
    \le \frac{2\sqrt{2}\log \left( \frac{2}{\delta} \right)}{\sqrt{n}},
    \end{aligned}
\end{equation*}
with probability at least $1-\delta$. \qed

\subsection{Proof of Lemma \ref{A_4}}\label{B.4}
To prove Lemma \ref{A_4}, we first introduce the McDiarmid's inequality.
\begin{lemma}
    McDiarmid's inequality. For $m$ random variables $X_i \in \mathcal{X}, i\in [m]$, assume $f:\mathcal{X}^m \to \mathbb{R}$ is the real function of $X_i$ and $\forall\ x_1,\cdots,x_m,x_i' \in \mathcal{X}$, we have
    \begin{equation*}
        \left| f\left( x_1,\cdots ,x_i,\cdots ,x_m \right) -f\left( x_1,\cdots ,x_{i}',\cdots ,x_m \right) \right|\le c_i,
    \end{equation*}
    then $\forall \ \epsilon >0$,the following inequality holds.
    \begin{equation*}
        P\left( f\left( X_1,\cdots ,X_m \right) -\mathbb{E}\left[ f\left( X_1,\cdots ,X_m \right) \right] \ge \epsilon \right) \le \exp \left\{ \frac{-2\epsilon ^2}{\sum_{i=1}^m{c_{i}^{2}}} \right\} .
    \end{equation*}
\end{lemma}
As mentioned in Lemma \ref{A_4}, $\hat{F}(\hat{Z};K^*)$ is a pairwise function and some tools in the i.i.d. condition is not satisfied, therefore, we make the following definition.
\begin{definition}
\textit{    (Rademacher complexity for $\hat{F}(\hat{Z};K^*)$) Let $\mathcal{H}$ is the function space of $\hat{z}$, the empirical Rademacher complexity of $\mathcal{L}$ is definied as}
    \begin{equation*}
        \hat{R}_n\left( \mathcal{L} \right) =\mathbb{E}_{\sigma}\left[ \underset{\hat{z}\in \mathcal{L}}{\text{sup}}|\frac{2}{\lfloor \frac{n}{2} \rfloor}\sum_{i=1}^{\lfloor \frac{n}{2} \rfloor}{\sigma _iK^*\left( x_i,x_{i+\lfloor \frac{n}{2} \rfloor} \right) \hat{z}\left( x_i \right) \hat{z}\left( x_{i+\lfloor \frac{n}{2} \rfloor} \right)} \right], 
    \end{equation*}
\end{definition}
where $\{\sigma_i\}_{i=1}^{\lfloor \frac{n}{2} \rfloor}$ are the i.i.d. Rademacher variables taking values $1$ and $-1$ with equal probability independent of the sample $S_n$. $\lfloor \frac{n}{2} \rfloor$ means the greatest integer less than or equal to $\frac{n}{2}$. The Rademacher complexity is the expectation of $\hat{R}_n(\mathcal{L})$, $R(\mathcal{L})=\mathbb{E}[\hat{R}_n(\mathcal{L})]$.

\textit{Proof}. Based on the definition of $\hat{Z}$, $\mathcal{C}$ can be reformulated as 
\begin{equation*}
    \begin{aligned}
        \mathcal{C}=F\left( \hat{Z};K^* \right) -\hat{F}\left( \hat{Z};K^* \right) &=\sum_{q=1}^k{\iint_{\mathcal{X}}{K^*\left( x,y \right) \hat{z}_q\left( x \right) \hat{z}_q\left( y \right)}\text{d}\rho \left( x \right) \text{d}\rho \left( y \right)}-\frac{1}{n^2}\sum_{q=1}^k{\sum_{i=1}^n{\sum_{j=1}^n{K^*\left( x_i,y_i \right) \hat{z}_q\left( x_i \right) \hat{z}_q\left( y_i \right)}}}
        \\
        &=\sum_{q=1}^k{\left( \mathbb{E}\left[ K^*\left( x,y \right) \hat{z}_q\left( x \right) \hat{z}_q\left( y \right) \right] -\hat{\mathbb{E}}\left[ K^*\left( x,y \right) \hat{z}_q\left( x \right) \hat{z}_q\left( y \right) \right] \right)}
        \\
        &\le k\sup_{\hat{z}_q \in \mathcal{L}} (\mathbb{E}[K^*(x,y)\hat{z}_q(x)\hat{z}_q(y)]-\hat{\mathbb{E}}[K^*(x,y)\hat{z}_q(x)\hat{z}_q(y)]). 
    \end{aligned}
\end{equation*}
Assume the i.i.d. sampled data are $S_n=\{x_1,\cdots,x_i,\cdots,x_n\}$ and $S_n^{i,x_i’}=\{x_1,\cdots,x_i',\cdots,x_n\}$, we have
\begin{equation*}
    \begin{aligned}
        | \underset{\hat{z}_q\in \mathcal{L}}{\text{sup}}( \mathbb{E}[ K^*( x,y ) \hat{z}_q( x ) \hat{z}_q( y ) ] -&\hat{\mathbb{E}}_{S_n}[ K^*( x,y ) \hat{z}_q( x ) \hat{z}_q( y ) ] ) -\underset{\hat{z}_q\in \mathcal{L}}{\text{sup}}( \mathbb{E}[ K^*( x,y ) \hat{z}_q( x ) \hat{z}_q( y ) ] -\hat{\mathbb{E}}_{S_{n}^{i,x_i'}}[ K^*( x,y ) \hat{z}_q( x ) \hat{z}_q( y ) ] ) | 
        \\
        \le &\underset{\hat{z}_q\in \mathcal{L}}{\text{sup}}\left| \hat{\mathbb{E}}_{S_n}\left[ K^*\left( x,y \right) \hat{z}_q\left( x \right) \hat{z}_q\left( y \right) \right] -\hat{\mathbb{E}}_{S_{n}^{i,x_i'}}\left[ K^*\left( x,y \right) \hat{z}_q\left( x \right) \hat{z}_q\left( y \right) \right] \right|
        \\
        \le &\frac{2}{n^2}\underset{\hat{z}_q\in \mathcal{L}}{\text{sup}}\sum_{j=1}^n{\left( \left| K^*\left( x_i,x_j \right) \hat{z}_q\left( x_i \right) \hat{z}_q\left( x_j \right) \right|+\left| K^*\left( x_{i}',x_j \right) \hat{z}_q\left( x_{i}' \right) \hat{z}_q\left( x_j \right) \right| \right)}
        \\
        \le &\frac{2}{n^2}\underset{\hat{z}_q\in \mathcal{L}}{\text{sup}}\sum_{j=1}^n{\left( \left| \hat{z}_q\left( x_i \right) \hat{z}_q\left( x_j \right) +\hat{z}_q\left( x_{i}' \right) \hat{z}_q\left( x_j \right) \right| \right)}
        \\
        \le &\frac{4}{n}.
    \end{aligned}
\end{equation*}

The first inequality arises because $\underset{x}{\text{sup}}\left( f\left( x \right) -g\left( x \right) \right) -\underset{x}{\text{sup}}\left( f\left( x \right) -h\left( x \right) \right) \le \underset{x}{\text{sup}}\left( h\left( x \right) -g\left( x \right) \right) $; the second inequality is readily derived from $\left| f\left( x \right) -g\left( x \right) \right|\le \left| f\left( x \right) \right|+\left| g\left( x \right) \right|$; concerning the third and fourth inequalities, we note that $K^*(x,y) \le 1$ and $\sum_{j=1}^n{\hat{z}_q\left( x_j \right) \hat{z}_q\left( x_j \right)}=n$. Therefore, by applying McDiarmid's inequality, we have
\begin{equation*}
    \begin{aligned}
        &\underset{\hat{z}_q\in \mathcal{L}}{\text{sup}}\left( \mathbb{E}\left[ K^*\left( x,y \right) \hat{z}_q\left( x \right) \hat{z}_q\left( y \right) \right] -\hat{\mathbb{E}}_{S_n}\left[ K^*\left( x,y \right) \hat{z}_q\left( x \right) \hat{z}_q\left( y \right) \right] \right) 
        \\
        \le &\mathbb{E}\left[ \underset{\hat{z}_q\in \mathcal{L}}{\text{sup}}\left( \mathbb{E}\left[ K^*\left( x,y \right) \hat{z}_q\left( x \right) \hat{z}_q\left( y \right) \right] -\hat{\mathbb{E}}_{S_n}\left[ K^*\left( x,y \right) \hat{z}_q\left( x \right) \hat{z}_q\left( y \right) \right] \right) \right] +\sqrt{\frac{8\log \frac{1}{\delta}}{n}},
    \end{aligned}
\end{equation*}
with probability at least $1-\delta$. Then we need to bound $\mathbb{E}\left[ \underset{\hat{z}_q\in \mathcal{L}}{\text{sup}}\left( \mathbb{E}\left[ K^*\left( x,y \right) \hat{z}_q\left( x \right) \hat{z}_q\left( y \right) \right] -\hat{\mathbb{E}}_{S_n}\left[ K^*\left( x,y \right) \hat{z}_q\left( x \right) \hat{z}_q\left( y \right) \right] \right) \right]$. According to \cite{918c02b9-94d7-38ea-9cf7-cb77ed046204}, we have

\begin{equation*}
    \begin{aligned}
        &\mathbb{E}\left[ \underset{\hat{z}_q\in \mathcal{L}}{\text{sup}}\left( \mathbb{E}\left[ K^*\left( x,y \right) \hat{z}_q\left( x \right) \hat{z}_q\left( y \right) \right] -\hat{\mathbb{E}}_{S_n}\left[ K^*\left( x,y \right) \hat{z}_q\left( x \right) \hat{z}_q\left( y \right) \right] \right) \right] 
        \\
        \le &\mathbb{E}\left[ \underset{\hat{z}_q\in \mathcal{L}}{\text{sup}}\left( \mathbb{E}\left[ K^*\left( x,y \right) \hat{z}_q\left( x \right) \hat{z}_q\left( y \right) \right] -\frac{1}{\lfloor \frac{n}{2} \rfloor}\sum_{i=1}^{\lfloor \frac{n}{2} \rfloor}{K^*\left( x_i,x_{\lfloor \frac{n}{2} \rfloor +i} \right) \hat{z}_q\left( x_i \right) \hat{z}_q\left( x_{\lfloor \frac{n}{2} \rfloor +i} \right)} \right) \right] 
    \end{aligned}
\end{equation*}
Donate $S_n'=\{x_1',\cdots,x_n'\}$ be the sampled i.i.d. data, $S_n'$ is independent of $S_n$. We have
\begin{equation*}
    \begin{aligned}
        &\mathbb{E}\left[ \underset{\hat{z}_q\in \mathcal{L}}{\text{sup}}\left( \mathbb{E}\left[ K^*\left( x,y \right) \hat{z}_q\left( x \right) \hat{z}_q\left( y \right) \right] -\frac{1}{\lfloor \frac{n}{2} \rfloor}\sum_{i=1}^{\lfloor \frac{n}{2} \rfloor}{K^*\left( x_i,x_{\lfloor \frac{n}{2} \rfloor +i} \right) \hat{z}_q\left( x_i \right) \hat{z}_q\left( x_{\lfloor \frac{n}{2} \rfloor +i} \right)} \right) \right] 
        \\
        =&\mathbb{E}_{S_n}\left[ \underset{\hat{z}_q\in \mathcal{L}}{\text{sup}}\left( \mathbb{E}_{S_{n}'}\left[ \frac{1}{\lfloor \frac{n}{2} \rfloor}\sum_{i=1}^{\lfloor \frac{n}{2} \rfloor}{K^*\left( x_{i}',x_{\lfloor \frac{n}{2} \rfloor +i}' \right) \hat{z}_q\left( x_{i}' \right) \hat{z}_q( x_{\lfloor \frac{n}{2} \rfloor +i}' )}-\frac{1}{\lfloor \frac{n}{2} \rfloor}\sum_{i=1}^{\lfloor \frac{n}{2} \rfloor}{K^*( x_i,x_{\lfloor \frac{n}{2} \rfloor +i} ) \hat{z}_q\left( x_i \right) \hat{z}_q\left( x_{\lfloor \frac{n}{2} \rfloor +i} \right)} \right] \right) \right] 
        \\
        \le &\mathbb{E}_{S_n,S_{n}'}\left[ \underset{\hat{z}_q\in \mathcal{L}}{\text{sup}}\left( \frac{1}{\lfloor \frac{n}{2} \rfloor}\sum_{i=1}^{\lfloor \frac{n}{2} \rfloor}{\left( K^*\left( x_{i}',x_{\lfloor \frac{n}{2} \rfloor +i}' \right) \hat{z}_q\left( x_{i}' \right) \hat{z}_q\left( x_{\lfloor \frac{n}{2} \rfloor +i}' \right) -K^*\left( x_i,x_{\lfloor \frac{n}{2} \rfloor +i} \right) \hat{z}_q\left( x_i \right) \hat{z}_q\left( x_{\lfloor \frac{n}{2} \rfloor +i} \right) \right)} \right) \right] 
        \\
        =&\mathbb{E}_{S_n,S_{n}',\sigma}\left[ \underset{\hat{z}_q\in \mathcal{L}}{\text{sup}}\left( \frac{1}{\lfloor \frac{n}{2} \rfloor}\sum_{i=1}^{\lfloor \frac{n}{2} \rfloor}{\sigma _i\left( K^*\left( x_{i}',x_{\lfloor \frac{n}{2} \rfloor +i}' \right) \hat{z}_q\left( x_{i}' \right) \hat{z}_q\left( x_{\lfloor \frac{n}{2} \rfloor +i}' \right) -K^*\left( x_i,x_{\lfloor \frac{n}{2} \rfloor +i} \right) \hat{z}_q\left( x_i \right) \hat{z}_q\left( x_{\lfloor \frac{n}{2} \rfloor +i} \right) \right)} \right) \right] 
        \\
        =&\frac{2}{\lfloor \frac{n}{2} \rfloor}\mathbb{E}_{S_{n}',\sigma}\left[ \underset{\hat{z}_q\in \mathcal{L}}{\text{sup}}\left( \sum_{i=1}^{\lfloor \frac{n}{2} \rfloor}{\sigma _iK^*\left( x_{i}',x_{\lfloor \frac{n}{2} \rfloor +i}' \right) \hat{z}_q\left( x_{i}' \right) \hat{z}_q\left( x_{\lfloor \frac{n}{2} \rfloor +i}' \right)} \right) \right] 
        \\
        \le &\frac{2}{\lfloor \frac{n}{2} \rfloor}\mathbb{E}_{S_{n}'}\left[ \left( \underset{\hat{z}_q\in \mathcal{L}}{\text{sup}}\sum_{i=1}^{\lfloor \frac{n}{2} \rfloor}{\left( K^*\left( x_{i}',x_{\lfloor \frac{n}{2} \rfloor +i}' \right) \hat{z}_q\left( x_{i}' \right) \hat{z}_q\left( x_{\lfloor \frac{n}{2} \rfloor +i}' \right) \right) ^2} \right) ^{\frac{1}{2}} \right],
    \end{aligned}
\end{equation*}
where $\{\sigma_i\}_{i=1}^{\lfloor \frac{n}{2} \rfloor}$ are the Rademacher variables. The second inequality is derived from Jensen's inequality; the third equality uses the standard symmetrization technique and the last inequality utilizes the Khinchin-Kahane inequality \cite{Lata}. Assume that $||\hat{z}_q||_{\infty} < \sqrt{c_0}$, we can obtain that
\begin{equation*}
    \begin{aligned}
        & \frac{2}{\lfloor \frac{n}{2} \rfloor}\mathbb{E}_{S_{n}'}\left[ \left( \underset{\hat{z}_q\in \mathcal{L}}{\text{sup}}\sum_{i=1}^{\lfloor \frac{n}{2} \rfloor}{\left( K^*\left( x_{i}',x_{\lfloor \frac{n}{2} \rfloor +i}' \right) \hat{z}_q\left( x_{i}' \right) \hat{z}_q\left( x_{\lfloor \frac{n}{2} \rfloor +i}' \right) \right) ^2} \right) ^{\frac{1}{2}} \right],
        \\
        \le & \frac{2}{\lfloor \frac{n}{2} \rfloor} C\sqrt{\lfloor \frac{n}{2} \rfloor}
        \\
        \le & \frac{2\sqrt{2}c_0}{\sqrt{n}}.
    \end{aligned}
\end{equation*}
Thus, we can obtain
\begin{equation*}
    \mathcal{C} = F(\hat{Z};K^*)-\hat{F}(\hat{Z};K^*) \le k\left( \frac{2\sqrt{2}c_0}{\sqrt{n}}+\sqrt{\frac{8\log \frac{1}{\delta}}{n}} \right).
\end{equation*}
with at least probability $1-\delta$.\qed

\subsection{Proof of Lemma \ref{A_5}} \label{B.5}
\textit{Proof}. For a given sequence $(a_q)_q$,  $a_q\mathbf{\hat{z}}_{q}$ and $b_q\boldsymbol{\hat{z}}_{q}$, we can always find another sequence $(b_q)_q$ such that $\cos \theta(a_q\hat{\mathbf{z}}_q,b_q \boldsymbol{\hat{z}}_q)\ge 0$. Therefore, without loss of generality, we assume that the angle between $\hat{\mathbf{z}}_q$ and $\hat{\boldsymbol{z}}_q$ is within $[0,\frac{\pi}{2}]$.
\begin{equation*}
        \lVert a_q\mathbf{\hat{z}}_{q}-b_q\boldsymbol{\hat{z}}_{q} \rVert _{\infty} \le \lVert a_q\mathbf{\hat{z}}_{q}-b_q\boldsymbol{\hat{z}}_{q} \rVert _2=\sqrt{2-2\cos \theta}=\sqrt{4\sin ^2\left( \frac{\theta}{2} \right)}=2\left| \sin \left( \frac{\theta}{2} \right) \right|.
\end{equation*}
With $\sin \left( \theta \right) =2\sin \left( \frac{\theta}{2} \right) \cos \left( \frac{\theta}{2} \right) $ and $\cos \left( \frac{\theta}{2} \right) \ge \frac{1}{\sqrt{2}}$, we have 
\begin{equation*}
    \lVert a_q\mathbf{\hat{z}}_{q}-b_q\boldsymbol{\hat{z}}_{q} \rVert _{\infty} \le \sqrt{2} \sin(\theta),
\end{equation*}
where $\theta = \theta(a_q\hat{\mathbf{z}}_q,b_q\hat{\boldsymbol{z}}_q)$. From the proof of Lemma \ref{A_2}, we can readily deduce that as $m,n\to \infty, m \gg \log n$, $\lVert a_q\mathbf{\hat{z}}_{q}-b_q\boldsymbol{\hat{z}}_{q} \rVert _{\infty} \to 0$. \qed

\subsection{Proof of Lemma \ref{A_6}} \label{B.6}
\textit{Proof}. For $\lVert \mathbf{K}^{\mathbf{w}}-\mathbf{K}^* \rVert _{\mathrm{F}}^{2}$, we have
 \begin{equation*}
    \begin{aligned}
        &\lVert \mathbf{K}^{\mathbf{w}}-\mathbf{K}^* \rVert _{\mathrm{F}}^{2}
        \\
        =&2\lVert \mathbf{K}^*-\mathbf{K}^{\mathbf{w}} \rVert _{\mathrm{F}}^{2}-\lVert \mathbf{K}^*-\mathbf{K}^{\mathbf{w}} \rVert _{\mathrm{F}}^{2}
        \\
        =&2\mathrm{tr}\left( \left( \mathbf{K}^*-\mathbf{K}^{\mathbf{w}} \right) ^{\top}\left( \mathbf{K}^*-\frac{1}{m}\sum_{t=1}^m{mw_t\mathbf{K}^{\left( t \right)}} \right) \right) -\lVert \mathbf{K}^*-\mathbf{K}^{\mathbf{w}} \rVert _{\mathrm{F}}^{2}
        \\
        =&\frac{1}{m}\sum_{t=1}^m{2\mathrm{tr}\left( \left( \mathbf{K}^*-\mathbf{K}^{\mathbf{w}} \right) ^{\top}\left( \mathbf{K}^*-mw_t\mathbf{K}^{\left( t \right)} \right) \right) -\lVert \mathbf{K}^*-\mathbf{K}^{\mathbf{w}} \rVert _{\mathrm{F}}^{2}}
        \\
        =&\frac{1}{m}\sum_{t=1}^m{\left( -2\mathrm{tr}\left( \left( \mathbf{K}^*-\mathbf{K}^{\mathbf{w}} \right) ^{\top}\left( mw_t\mathbf{K}^{\left( t \right)}-\mathbf{K}^* \right) \right) -\lVert \mathbf{K}^*-\mathbf{K}^{\mathbf{w}} \rVert _{\mathrm{F}}^{2} \right)}
        \\
        =&\frac{1}{m}\sum_{t=1}^m{\left( -2\mathrm{tr}\left( \left( \mathbf{K}^*-\mathbf{K}^{\mathbf{w}} \right) ^{\top}\left( mw_t\mathbf{K}^{\left( t \right)}-\mathbf{K}^* \right) \right) -\lVert \mathbf{K}^*-\mathbf{K}^{\mathbf{w}} \rVert _{\mathrm{F}}^{2} \right)}
        \\
        +&\frac{1}{m}\sum_{t=1}^m{\left( -\lVert mw_t\mathbf{K}^{\left( t \right)}-\mathbf{K}^* \rVert _{\mathrm{F}}^{2}+\lVert mw_t\mathbf{K}^{\left( t \right)}-\mathbf{K}^* \rVert _{\mathrm{F}}^{2} \right)}
        \\
        =&\frac{1}{m}\sum_{t=1}^m{\left( -\lVert \mathbf{K}^*-\mathbf{K}^{\mathbf{w}}+mw_t\mathbf{K}^{\left( t \right)}-\mathbf{K}^* \rVert _{\mathrm{F}}^{2}+\lVert mw_t\mathbf{K}^{\left( t \right)}-\mathbf{K}^* \rVert _{\mathrm{F}}^{2} \right)}
        \\
        =&\frac{1}{m}\sum_{t=1}^m{\lVert mw_t\mathbf{K}^{\left( t \right)}-\mathbf{K}^* \rVert _{\mathrm{F}}^{2}}-\frac{1}{m}\sum_{t=1}^m{\lVert mw_t\mathbf{K}^{\left( t \right)}-\mathbf{K}^{\mathbf{w}} \rVert _{\mathrm{F}}^{2}}.
    \end{aligned}
\end{equation*}
This concludes the proof of Lemma \ref{A_6}. \qed

\subsection{Proof of Lemma \ref{A_7}} \label{B.7}
\textit{Proof}. Based on Lemma \ref{A2}, we have
\begin{equation*}
    \begin{aligned}
    &\hat{F}\left( \mathbf{\hat{Z}};\mathbf{K}^{\mathbf{w}} \right) -\hat{F}\left( \hat{\boldsymbol{\mathcal{Z}}};\mathbf{K}^* \right) 
    \\
    = &\frac{1}{n}\text{tr}\left( \mathbf{\hat{Z}}^{\top}\mathbf{K}^{\mathbf{w}}\mathbf{\hat{Z}} \right) -\frac{1}{n}\text{tr}\left( \hat{\boldsymbol{\mathcal{Z}}}^{\top}\mathbf{K}^*\hat{\boldsymbol{\mathcal{Z}}} \right) 
    \\
    = &\frac{1}{n}\left( \text{tr}\left( \mathbf{\hat{Z}}^{\top}\left( \mathbf{K}^{\mathbf{w}}-\mathbf{K}^* \right) \mathbf{\hat{Z}} \right) \right) +\frac{1}{n}\text{tr}\left( \mathbf{K}^*\left( \mathbf{\hat{Z}\hat{Z}}^{\top}-\hat{\boldsymbol{\mathcal{Z}}}\hat{\boldsymbol{\mathcal{Z}}}^{\top} \right) \right) 
    \\
    \le &\frac{1}{n}\lVert \mathbf{K}^{\mathbf{w}}-\mathbf{K}^* \rVert _{\text{F}}\lVert \mathbf{\hat{Z}} \rVert _{\text{F}}^{2}+\frac{1}{n}\lVert \mathbf{K}^* \rVert _{\text{F}}\lVert \mathbf{\hat{Z}\hat{Z}}^{\top}-\hat{\boldsymbol{\mathcal{Z}}}\hat{\boldsymbol{\mathcal{Z}}}^{\top} \rVert _{\text{F}}
    \\
    \le &\frac{k}{n}\lVert \mathbf{K}^{\mathbf{w}}-\mathbf{K}^* \rVert _{\text{F}}+\frac{n}{n}\lVert \mathbf{\hat{Z}\hat{Z}}^{\top}-\hat{\boldsymbol{\mathcal{Z}}}\hat{\boldsymbol{\mathcal{Z}}}^{\top} \rVert _{\text{F}}
    \\
    \le &\frac{k}{n}\lVert \mathbf{K}^{\mathbf{w}}-\mathbf{K}^* \rVert _{\text{F}}+\frac{2\sqrt{2}}{\lambda _k\left( \mathbf{K}^{*} \right) -\lambda _{k+1}\left( \mathbf{K}^{*} \right)}\lVert \mathbf{K}^{\mathbf{w}}-\mathbf{K}^* \rVert _2
    \\
    \le &\left( \frac{k}{n}+\frac{2\sqrt{2}}{\lambda _k\left( \mathbf{K}^{*} \right) -\lambda _{k+1}\left( \mathbf{K}^{*} \right)} \right) \lVert \mathbf{K}^{\mathbf{w}}-\mathbf{K}^* \rVert _{\text{F}}
    \\
    \end{aligned}
\end{equation*}
The first inequality utilizes the properties of the matrix trace, and the second inequality holds because $\hat{\mathbf{Z}}$ is an $n\times k$ column-orthogonal matrix, and $||\mathbf{K}^*||_{\mathrm{F}} \le n||\mathbf{K}^*||_2 \le n$ (noting that $\mathbf{K}^*$ is a degree normalized matrix).
This concludes the proof of Lemma \ref{A_7}. \qed

\subsection{Proof of Lemma \ref{A_8}} \label{B.8}
\textit{Proof}. Note that $\tilde{w}=mw_t$ and $\mathbf{K}^\mathbf{w} = \sum_{t=1}^m w_t\mathbf{K}^{(t)}$, we have
\begin{equation*}
    \begin{aligned}
        &\underset{\mathbf{w}}{\min}\,\,\sum_{t=1}^m{\lVert \tilde{w}_t\mathbf{K}^{\left( t \right)}-\mathbf{K}^* \rVert _{\text{F}}^{2}}-\sum_{t=1}^m{\lVert \tilde{w}_t\mathbf{K}^{\left( t \right)}-\mathbf{K}^{\mathbf{w}} \rVert _{\text{F}}^{2}}
        \\
        \Leftrightarrow &\underset{\mathbf{w}}{\min}\,\,\sum_{t=1}^m{\left( \lVert mw_t\mathbf{K}^{\left( t \right)} \rVert _{\text{F}}^{2}-2\text{tr}\left( mw_t\mathbf{K}^{\left( t \right)}\mathbf{K}^* \right) +\lVert \mathbf{K}^* \rVert _{\text{F}}^{2} \right)}-\sum_{t=1}^m{\left( \lVert mw_t\mathbf{K}^{\left( t \right)} \rVert _{\text{F}}^{2}-2\text{tr}\left( mw_t\mathbf{K}^{\left( t \right)}\mathbf{K}^{\mathbf{w}} \right) +\lVert \mathbf{K}^{\mathbf{w}} \rVert _{\text{F}}^{2} \right)}
        \\
        \Leftrightarrow &\underset{\mathbf{w}}{\min}\,\,\sum_{t=1}^m{\left( -2\text{tr}\left( mw_t\mathbf{K}^{\left( t \right)}\mathbf{K}^* \right) +\lVert \mathbf{K}^* \rVert _{\text{F}}^{2} \right)}-\sum_{t=1}^m{\left( -2\text{tr}\left( mw_t\mathbf{K}^{\left( t \right)}\mathbf{K}^{\mathbf{w}} \right) +\lVert \mathbf{K}^{\mathbf{w}} \rVert _{\text{F}}^{2} \right)}
        \\
        \Leftrightarrow &\underset{\mathbf{w}}{\min}\,\,-2m\text{tr}\left( \mathbf{K}^{\mathbf{w}}\mathbf{K}^* \right) -\sum_{t=1}^m{\left( -2\text{tr}\left( mw_t\mathbf{K}^{\left( t \right)}\mathbf{K}^{\mathbf{w}} \right) +\lVert \mathbf{K}^{\mathbf{w}} \rVert _{\text{F}}^{2} \right)}
        \\
        \Leftrightarrow &\underset{\mathbf{w}}{\min}\,\,-2m\text{tr}\left( \mathbf{K}^{\mathbf{w}}\mathbf{K}^* \right) +2m\text{tr}\left( \mathbf{K}^{\mathbf{w}}\mathbf{K}^{\mathbf{w}} \right) -m\text{tr}\left( \mathbf{K}^{\mathbf{w}}\mathbf{K}^{\mathbf{w}} \right) 
        \\
        \Leftrightarrow &\underset{\mathbf{w}}{\min}\,\,-2\text{tr}\left( \mathbf{K}^{\mathbf{w}}\mathbf{K}^* \right) +\text{tr}\left( \mathbf{K}^{\mathbf{w}}\mathbf{K}^{\mathbf{w}} \right) 
    \end{aligned}
\end{equation*}
Note that in the proof, since ignoring $\mathbf{K}^*$ does not affect the optimization of $\mathbf{w}$, we have omitted term $||\mathbf{K}^*||_\mathrm{F}$. This concludes the proof of Lemma \ref{A_8}. \qed

\newpage
\section{Pseudo code for Section \ref{optimization12}} \label{pseudo}
\begin{algorithm}[!htbp]
    \caption{}
    \renewcommand{\algorithmicrequire}{\textbf{Input:}}
    \renewcommand{\algorithmicensure}{\textbf{Output:}}
    \newcommand{\algorithmicinitialize}{\textbf{Initialization:}}
    \newcommand{\INITIALIZE}{\item[\algorithmicinitialize]}
    
    \begin{algorithmic}[1]
        \REQUIRE Base clusterings $\{\mathbf{K}^{(t)}\}_{t=1}^m$.
        \INITIALIZE Weight $\{w_t\}_t^m=\frac{1}{m}$, the number of cluster $k$, hyper-parameter $\alpha$, the number of iterations $p$.
        \ENSURE Clustering result.
        \WHILE {not converged}
            \STATE Compute the matrix $\mathbf{Z}^{(p)}$ (consists of the eigenvectors corresponding to the top $k$ largest eigenvalues of $\mathbf{K}^{\mathbf{w}^{(p)}}$).
            \STATE Calculate $\frac{\partial\mathcal{J}(\mathbf{w}^{(p)})}{\partial w_t}$ and the descent direction $d_t^{(p)}$ in Eq. (\ref{d_t}).
            \STATE Update $\mathbf{w}^{(p+1)}$ as $\mathbf{w}^{(p+1)}\leftarrow \mathbf{w}^{(p)} + \beta \mathbf{d}^{(p)}$.
            \IF{$|\mathbf{w}^{(p+1)} - \mathbf{w}^{(p)}|\le \epsilon$}
            \STATE Break.
            \ENDIF
            \STATE $p\leftarrow p+1$.
        \ENDWHILE
        \STATE Apply the $k$-means algorithm to $\mathbf{Z}^{(p)}$ to obtain the final clustering result.
    \end{algorithmic}
\end{algorithm}

\section{Related Work}
This section reviews related works on generalization performance of ensemble clustering. In \cite{7811216}, the author derived the generalization error bound of ensemble clustering with finite base clusterings from the perspective of weighted kernel $k$-means. Denote $\mathbf{B}_{n\times (\sum_{t=1}^m k^{(t)})}$ as a combined binary matrix of $m$ base clusterings where
\begin{align}
    &\mathbf{B}(x,\cdot) = b(x)=<b(x)_1,\cdots,b(x)_m>, \nonumber  \\
    &b(x)_t = <b(x)_{t1},\cdots,b(x)_{tk^{(t)}}>, \nonumber \\
    &b\left( x \right) _{ti}\begin{cases}
	1,&		\text{if}\quad \pi ^{\left( t \right)}\left( x \right) =i  \\
	0,&		\text{else} \nonumber \\
    \end{cases}.
\end{align}
Based on the above definition, the author derived ensemble clustering is equivalent to weighted kernel $k$-means algorithm,
\begin{align}
    \hat{F}\left( \mathbf{\hat{Z}} \right) =\underset{\mathbf{\hat{Z}}}{\max}\frac{1}{k}\text{tr}\left( \mathbf{Z}^{\top}\mathbf{D}^{-1/2}\mathbf{SD}^{-1/2}\mathbf{Z} \right) \Leftrightarrow \hat{G}\left( x \right) =\sum_{x\in \mathcal{X}}{g_{m_1,\cdots ,m_k}\left( x \right)}, \nonumber
\end{align}
where $\mathbf{S}$ is the CA matrix, $g_{m_1,\cdots ,m_k}(x)=\min_k ||\frac{b(x)}{w_b(x)}-m_k||^2$, $m_k=\frac{\sum_{x\in C_k}b(x)}{\sum_{x\in C_k}w_{b(x)}}$, and $w_b(x)=\mathbf{D}(x,x)=\sum_{t=1}^m\sum_{i=1}^n \delta (\pi^{(t)}(x),\pi^{(t)}(x_i)).$ The generalization error bound of ensemble clustering with finite base clusterings is
\begin{align}
    \mathbb{E}_xg_{m_1,\cdots ,m_k}\left( x \right) &-\frac{1}{n}\sum_{i=1}^n{g_{m_1,\cdots ,m_k}\left( x_i \right)} \nonumber
    \\
    \le \frac{\sqrt{2\pi}mk}{n}\left( \sum_{i=1}^n{\left( w_{b\left( x_i \right)} \right) ^{-2}} \right) ^{\frac{1}{2}}+\frac{\sqrt{8\pi}mk}{\sqrt{n}\min _{x\in \mathcal{X}}w_{b\left( x \right)}}&+\frac{\sqrt{2\pi}mk}{n\min _{x\in \mathcal{X}}\left( w_{b\left( x \right)} \right) ^2}\left( \sum_{i=1}^n{\left( w_{b\left( x_i \right)} \right) ^2} \right) ^{\frac{1}{2}}+\left( \frac{\ln \left( 1/\delta \right)}{2n} \right) ^{\frac{1}{2}}, \nonumber
\end{align}
with probability $1-\delta$. To the best of our knowledge, this work is the only one that provides a generalization error bound for ensemble clustering. Other theoretical analyses related to clustering include the generalization performance of multi-view clustering. For example, \cite{9857664} proposed SimpleMKKM algorithm in multi-view clustering and derived its generalization error. \cite{pmlr-v202-liang23b} demonstrated the consistency of kernel weights in multi-view clustering and derived its the excess risk bound. Nevertheless, the scenarios they consider are remain limited to finite ensembles.

\section{Comparative Experiment}
In this section, we provide additional details about the comparative experiments that are omitted in the main text due to space limitations.

\subsection{Details of Datasets}\label{DeDa}
In the comparative experiments in Section \ref{ComExp}, we used 10 benchmark datasets including images, DNA, sensor information, etc. We have summarized the feature information of the datasets in Table \ref{tab:datasets}, and the detailed information is as follows:
\begin{enumerate}
    \item \textbf{Phishing Websites}\footnote{\url{http://archive.ics.uci.edu/dataset/327/phishing+websites}}: The dataset consists of a collection of legitimate and phishing website instances. Each website is represented by the set of features which denote, whether the website is legitimate or not.
    \item \textbf{Rice}\footnote{\url{http://archive.ics.uci.edu/dataset/545/rice+cammeo+and+osmancik}}: A total of 3810 images of rice grains were captured from two species: Cammeo and Osmancik rices. For each grain in these images, seven morphological features were extracted.
    \item \textbf{TOX\_171}\footnote{\url{https://github.com/jundongl/scikit-feature/blob/master/skfeature/data/TOX-171.mat}}: The dataset contains 171 samples, each with 5748 features, derived from feature selection at Arizona State University’s repository.
    \item \textbf{Obesity}\footnote{\url{https://archive.ics.uci.edu/dataset/544/estimation+of+obesity+levels+based+on+eating+habits+and+physical+condition}}: The dataset contains 2111 instances from individuals in the countries of Mexico, Peru, and Colombia. It includes 16 features reflecting eating habits and physical conditions, designed to estimate obesity levels.
    \item \textbf{Seeds}\footnote{\url{https://archive.ics.uci.edu/dataset/236/seeds}}: The dataset includes measurements of the geometrical properties of kernels from three wheat varieties, with seven real-valued features extracted using a soft X-ray technique and the GRAINS package.
    \item \textbf{ALLAML}\footnote{\url{https://github.com/jundongl/scikit-feature/blob/master/skfeature/data/ALLAML.mat}}: The dataset consists of a DNA microarray data matrix, where rows represent genes and columns represent cancer patients diagnosed with one of two types of leukemia: AML or ALL. The elements of the matrix indicate gene expression levels in the corresponding patients.
    \item \textbf{warpAR10P}\footnote{\url{https://github.com/jundongl/scikit-feature/blob/master/skfeature/data/warpAR10P.mat}}: The dataset includes over 4000 color images of 126 individuals, comprising 70 men and 56 women. It captures various facial expressions, lighting conditions, and occlusions.
    \item \textbf{WFRN}\footnote{\url{https://archive.ics.uci.edu/dataset/194/wall+following+robot+navigation+data}}: The dataset includes four sensor readings, termed “simplified distances” (\textit{i.e.} front, left, right and back). Each distance represents the minimum sensor reading within a 60-degree arc in the respective direction around the robot.
    \item \textbf{Abalone}\footnote{\url{https://archive.ics.uci.edu/dataset/1/abalone}}: The dataset is designed to predict the age of abalones by collecting eight physical measurements, including sex, length, diameter, height, whole weight, shucked weight, viscera weight and shell weight.
    \item \textbf{Website Phishing}\footnote{\url{https://archive.ics.uci.edu/dataset/379/website+phishing}}: The dataset includes 1353 websites, with phishing URLs sourced from the Phishtank data archive and legitimate websites collected from Yahoo and starting point directories using a custom PHP web script. It comprises 548 legitimate websites, 702 phishing URLs and 103 suspicious URLs.
\end{enumerate}

\begin{table}[htbp]
    \centering
    \caption{Size of different datasets}
    \label{tab:datasets}
        \begin{tabular}{cccccc}
            \toprule
            \textbf{No.} & \textbf{Dataset} & \textbf{\#Instance} & \textbf{\#Feature} & \textbf{\#Class} \\ 
            \midrule
            D1 & Phishing Websites & 2456 & 30 & 2 \\
            D2 & Rice & 3810 & 7 & 2 \\
            D3 & TOX\_171 & 171 & 5748 & 4 \\
            D4 & Obesity & 2111 & 16 & 7 \\
            D5 & Seeds & 210 & 7 & 3 \\
            D6 & ALLAML & 72 & 7129 & 2 \\
            D7 & warpAR10P & 130 & 2400 & 10 \\
            D8 & WFRN & 5456 & 4 & 4 \\
            D9 & Abalone & 4177 & 8 & 3 \\
            D10 & Website Phishing & 1353 & 9 & 3 \\ 
            \bottomrule
    \end{tabular}
\end{table}

\begin{table}[!ht]
    \centering
    \setlength{\tabcolsep}{3.8pt}
    \caption{Performance (\%) evaluation of different datasets based on the ARI metric. We have highlighted the values of the best-performing method in \textbf{bold}, and the second-best method is marked with an \underline{underline}.}
    \begin{tabular}{c|cccccccccc|c}
        \toprule
        Method & D1 & D2 & D3 & D4 & D5 & D6 & D7 & D8 & D9 & D10 & Average  \\ 
        \midrule
        CEAM (TKDE'24) & 6.6\tiny{$\pm$12} & 42.8\tiny{$\pm$31} & 12.9\tiny{$\pm$4} & 20.4\tiny{$\pm$1} & 59.0\tiny{$\pm$13} & 2.7\tiny{$\pm$5} & 2.5\tiny{$\pm$1} & 10.8\tiny{$\pm$4} & 12.8\tiny{$\pm$5} & 10.1\tiny{$\pm$7} & 18.1\tiny{$\pm$8}  \\ 
        CEs$^2$L (AIJ'19) & 2.4\tiny{$\pm$4} & 3.0\tiny{$\pm$10} & 14.0\tiny{$\pm$3} & 20.3\tiny{$\pm$2} & 33.3\tiny{$\pm$19} & 18.3\tiny{$\pm$6} & 0.2\tiny{$\pm$2} & 6.8\tiny{$\pm$7} & 15.4\tiny{$\pm$4} & 9.6\tiny{$\pm$9} & 12.3\tiny{$\pm$7}  \\ 
        CEs$^2$Q (AIJ'19) & 1.7\tiny{$\pm$3} & 3.5\tiny{$\pm$7} & 12.4\tiny{$\pm$3} & 20.0\tiny{$\pm$2} & 31.2\tiny{$\pm$17} & 18.5\tiny{$\pm$6} & 0.3\tiny{$\pm$2} & 9.0\tiny{$\pm$4} & 15.2\tiny{$\pm$5} & 6.7\tiny{$\pm$5} & 11.8\tiny{$\pm$6}  \\ 
        LWEA (TCYB'18) & -0.5\tiny{$\pm$0} & 62.9\tiny{$\pm$4} & 13.1\tiny{$\pm$3} & 21.2\tiny{$\pm$1} & 57.5\tiny{$\pm$5} & 18.5\tiny{$\pm$6} & 0.0\tiny{$\pm$2} & 10.0\tiny{$\pm$4} & 13.5\tiny{$\pm$3} & 8.8\tiny{$\pm$6} & 20.5\tiny{$\pm$4}  \\ 
        NWCA (arXiv'24) & -0.5\tiny{$\pm$0} & 62.3\tiny{$\pm$4} & 12.9\tiny{$\pm$2} & 21.6\tiny{$\pm$1} & 56.3\tiny{$\pm$6} & 19.8\tiny{$\pm$5} & -0.1\tiny{$\pm$2} & 10.4\tiny{$\pm$1} & 13.3\tiny{$\pm$3} & 11.7\tiny{$\pm$6} & 20.8\tiny{$\pm$3}  \\ 
        ECCMS (TNNLS'24) & -0.5\tiny{$\pm$0} & 56.1\tiny{$\pm$24} & 13.5\tiny{$\pm$3} & 21.3\tiny{$\pm$1} & 60.8\tiny{$\pm$7} & 19.0\tiny{$\pm$6} & -0.3\tiny{$\pm$1} & \underline{12.2}\tiny{$\pm$4} & 14.0\tiny{$\pm$3} & 10.5\tiny{$\pm$6} & 20.7\tiny{$\pm$6}  \\ 
        MKKM (arXiv'18) & 8.8\tiny{$\pm$14} & 47.1\tiny{$\pm$25} & 9.5\tiny{$\pm$2} & 14.2\tiny{$\pm$5} & 53.8\tiny{$\pm$10} & 13.6\tiny{$\pm$12} & 2.1\tiny{$\pm$2} & 7.2\tiny{$\pm$3} & 10.9\tiny{$\pm$6} & 10.1\tiny{$\pm$7} & 17.7\tiny{$\pm$8}  \\ 
        SMKKM (TPAMI'23) & 8.8\tiny{$\pm$5} & 41.9\tiny{$\pm$10} & 14.6\tiny{$\pm$3} & 17.0\tiny{$\pm$3} & 55.5\tiny{$\pm$11} & 13.2\tiny{$\pm$9} & \underline{3.5}\tiny{$\pm$1} & 7.2\tiny{$\pm$4} & \underline{15.7}\tiny{$\pm$2} & 12.2\tiny{$\pm$5} & 19.0\tiny{$\pm$5}  \\ 
        SEC (TKDE'17) & 8.9\tiny{$\pm$15} & 23.8\tiny{$\pm$25} & 12.8\tiny{$\pm$4} & 13.5\tiny{$\pm$5} & 26.9\tiny{$\pm$19} & 13.5\tiny{$\pm$12} & 1.1\tiny{$\pm$2} & 5.6\tiny{$\pm$7} & 7.2\tiny{$\pm$6} & 5.2\tiny{$\pm$5} & 11.9\tiny{$\pm$9}  \\ 
        \midrule
        Fix $\alpha=0.1$ & \underline{30.8}\tiny{$\pm$15} & \underline{69.2}\tiny{$\pm$1} & \underline{15.8}\tiny{$\pm$4} & \underline{22.1}\tiny{$\pm$2} & \underline{67.5}\tiny{$\pm$5} & \underline{20.6}\tiny{$\pm$5} & 2.6\tiny{$\pm$1} & 12.0\tiny{$\pm$5} & 14.8\tiny{$\pm$5} & \underline{14.5}\tiny{$\pm$6} & \underline{27.0}\tiny{$\pm$4}  \\ 
        Proposed & \textbf{30.8}\tiny{$\pm$15} & \textbf{69.5}\tiny{$\pm$2} & \textbf{16.7}\tiny{$\pm$3} & \textbf{22.1}\tiny{$\pm$2} & \textbf{67.5}\tiny{$\pm$5} & \textbf{21.5}\tiny{$\pm$5} & \textbf{4.1}\tiny{$\pm$1} & \textbf{18.4}\tiny{$\pm$2} & \textbf{16.0}\tiny{$\pm$3} & \textbf{14.5}\tiny{$\pm$6} & \textbf{28.1}\tiny{$\pm$3}  \\ 
        \toprule
    \end{tabular}
    \label{table:ARI}
\end{table}

\begin{table}[!ht]
    \centering
    \setlength{\tabcolsep}{3.8pt}
    \caption{Performance (\%) evaluation of different datasets based on the F-score metric. We have highlighted the values of the best-performing method in \textbf{bold}, and the second-best method is marked with an \underline{underline}.}
    \begin{tabular}{c|cccccccccc|c}
        \toprule
        Method & D1 & D2 & D3 & D4 & D5 & D6 & D7 & D8 & D9 & D10 & Average  \\ 
        \midrule
        CEAM (TKDE'24) & 60.5\tiny{$\pm$9} & 79.4\tiny{$\pm$15} & 46.4\tiny{$\pm$3} & 42.7\tiny{$\pm$1} & 83.1\tiny{$\pm$7} & 66.0\tiny{$\pm$2} & 22.8\tiny{$\pm$2} & 51.4\tiny{$\pm$3} & 49.9\tiny{$\pm$3} & 61.8\tiny{$\pm$6} & 56.4\tiny{$\pm$5}  \\ 
        CEs$^2$L (AIJ'19) & 58.0\tiny{$\pm$4} & 59.5\tiny{$\pm$7} & 46.3\tiny{$\pm$2} & 42.0\tiny{$\pm$2} & 65.0\tiny{$\pm$12} & 72.2\tiny{$\pm$3} & 19.3\tiny{$\pm$2} & 49.1\tiny{$\pm$5} & 51.7\tiny{$\pm$3} & 62.7\tiny{$\pm$6} & 52.6\tiny{$\pm$4}  \\ 
        CEs$^2$Q (AIJ'19) & 57.4\tiny{$\pm$3} & 60.3\tiny{$\pm$5} & 44.7\tiny{$\pm$3} & 41.9\tiny{$\pm$2} & 62.9\tiny{$\pm$12} & 72.4\tiny{$\pm$3} & 19.2\tiny{$\pm$1} & 50.5\tiny{$\pm$5} & 51.6\tiny{$\pm$3} & 60.3\tiny{$\pm$4} & 52.1\tiny{$\pm$4}  \\ 
        LWEA (TCYB'18) & 55.5\tiny{$\pm$0} & 89.6\tiny{$\pm$1} & 46.0\tiny{$\pm$3} & 43.2\tiny{$\pm$1} & 81.7\tiny{$\pm$4} & 72.4\tiny{$\pm$3} & 18.6\tiny{$\pm$2} & 49.5\tiny{$\pm$1} & 51.3\tiny{$\pm$2} & 61.2\tiny{$\pm$4} & 56.9\tiny{$\pm$2}  \\ 
        NWCA (arXiv'24) & 55.5\tiny{$\pm$0} & 89.4\tiny{$\pm$1} & 45.9\tiny{$\pm$2} & 43.6\tiny{$\pm$1} & 80.7\tiny{$\pm$5} & 73.2\tiny{$\pm$2} & 18.8\tiny{$\pm$2} & 49.2\tiny{$\pm$1} & 51.2\tiny{$\pm$2} & 63.5\tiny{$\pm$4} & 57.1\tiny{$\pm$2}  \\ 
        ECCMS (TNNLS'24) & 55.5\tiny{$\pm$0} & 85.6\tiny{$\pm$12} & 46.1\tiny{$\pm$3} & 43.3\tiny{$\pm$1} & 84.0\tiny{$\pm$3} & 72.6\tiny{$\pm$3} & 18.5\tiny{$\pm$2} & 51.0\tiny{$\pm$3} & 51.6\tiny{$\pm$3} & 62.5\tiny{$\pm$4} & 57.1\tiny{$\pm$3}  \\ 
        MKKM (arXiv'18) & 62.1\tiny{$\pm$10} & 82.6\tiny{$\pm$11} & 42.9\tiny{$\pm$3} & 37.4\tiny{$\pm$5} & 79.8\tiny{$\pm$7} & 70.8\tiny{$\pm$5} & \underline{25.2}\tiny{$\pm$3} & 50.2\tiny{$\pm$2} & 49.7\tiny{$\pm$6} & 62.5\tiny{$\pm$6} & 56.3\tiny{$\pm$6}  \\ 
        SMKKM (TPAMI'23) & 62.9\tiny{$\pm$4} & 73.7\tiny{$\pm$7} & 47.7\tiny{$\pm$3} & 39.8\tiny{$\pm$2} & 80.6\tiny{$\pm$8} & 69.9\tiny{$\pm$4} & 23.4\tiny{$\pm$3} & 53.2\tiny{$\pm$1} & \underline{52.2}\tiny{$\pm$1} & 63.3\tiny{$\pm$4} & 56.7\tiny{$\pm$4}  \\ 
        SEC (TKDE'17) & 62.2\tiny{$\pm$10} & 71.9\tiny{$\pm$12} & 46.0\tiny{$\pm$3} & 37.2\tiny{$\pm$4} & 59.9\tiny{$\pm$13} & 71.0\tiny{$\pm$4} & 20.5\tiny{$\pm$2} & 48.2\tiny{$\pm$5} & 45.7\tiny{$\pm$5} & 58.8\tiny{$\pm$5} & 52.1\tiny{$\pm$6}  \\ 

        \midrule
        Fix $\alpha=0.1$ & \underline{76.5}\tiny{$\pm$9} & \underline{91.6}\tiny{$\pm$0} & \underline{48.9}\tiny{$\pm$3} & \underline{43.7}\tiny{$\pm$1} & \underline{87.6}\tiny{$\pm$2} & \underline{73.3}\tiny{$\pm$3} & 21.4\tiny{$\pm$2} & \underline{55.4}\tiny{$\pm$5} & 51.5\tiny{$\pm$4} & \underline{65.1}\tiny{$\pm$5} & \underline{61.5}\tiny{$\pm$3}  \\ 
        Proposed & \textbf{76.5}\tiny{$\pm$9} & \textbf{91.7}\tiny{$\pm$1} & \textbf{49.8}\tiny{$\pm$2} & \textbf{43.7}\tiny{$\pm$1} & \textbf{87.6}\tiny{$\pm$2} & \textbf{73.8}\tiny{$\pm$2} & \textbf{27.3}\tiny{$\pm$3} & \textbf{63.3}\tiny{$\pm$1} & \textbf{52.3}\tiny{$\pm$2} & \textbf{65.1}\tiny{$\pm$5} & \textbf{63.1}\tiny{$\pm$3}  \\ 
        \toprule
    \end{tabular}
    \label{table:Purity}
\end{table}

\subsection{Details Method}\label{DeMe}
The detailed descriptions of 9 comparison methods introduced in Section \ref{ComExp} are as follows.
\begin{enumerate}
    \item CEAM  \cite{10238807}, this method introduces a novel approach for clustering ensemble which refines weak base clustering results through diffusion on an adaptive multiplex structure.
    \item CEs$^2$L, CEs$^2$Q \cite{LI201937}, these two methods use a linear determinacy function and a quadratic determinacy function to assess sample stability in clustering ensemble respectively, distinguishing stable samples (cluster core) from less stable ones (cluster halo) for robust clustering.
    \item LWEA \cite{huang2017locally}, this method enhances ensemble clustering by employing a local weighting strategy based on cluster uncertainty and an ensemble-driven validity measure.
    \item NWCA \cite{zhang2024similarity}, this method discovers that smaller clusters have higher precision and proposes the normalized ensemble entropy to weight different clusters accordingly.
    \item ECCMS \cite{jia2023ensemble}, this method enhances co-association matrices in ensemble clustering by extracting high-confidence pairings from base clusterings and propagating them to refine the CA matrix.
    \item MKKM \cite{bang2018robust}, this method utilizes a min-max model to manage adversarial perturbations, ensuring the identification of accurate clusterings by optimally balancing the influence of multiple data views.
    \item SMKKM \cite{9857664}, this method transforms a complex min-max problem into a simpler minimization of an optimal value function, optimizing kernel coefficients and clustering matrices effectively to achieve robust clustering performance.
    \item SEC \cite{7811216}, this method combines the strengths of the co-association matrix with the efficiency of weighted K-means clustering and derives its generalization error bound.
\end{enumerate}

\subsection{Details of Comparative Experiment}\label{DeCom}
    In the appendix, we continue to demonstrate the performance of the algorithm on ARI and Purity. As can be seen in Table \ref{table:ARI} and Table \ref{table:Purity}, on both ARI and Purity, our method consistently leads against the compared methods across all datasets. For example, on the D1 (Phishing) dataset, our method achieves an ARI of 30.8\%, while the second-best method only reaches 8.9\%; in terms of Purity, ours is at 76.5\%, whereas the second-best is at 62.9\%. Moreover, even with fixed hyper-parameter, our method outperforms others on these two metrics, and while it may not be the second-best method on some datasets, such as D8 (WFRN), it is only slightly weaker than the second-best method (with a 0.2\% difference in ARI).

\begin{figure*}[ht]
    \centering
    \subfigure[Phishing]{
    \includegraphics[width=0.18\linewidth]{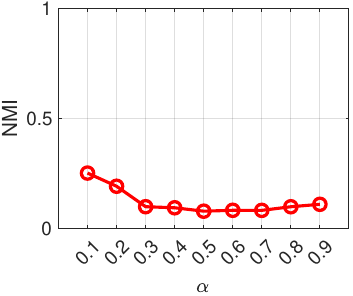}}
            \hspace{0.001\linewidth}
    \subfigure[Rice]{
    \includegraphics[width=0.18\linewidth]{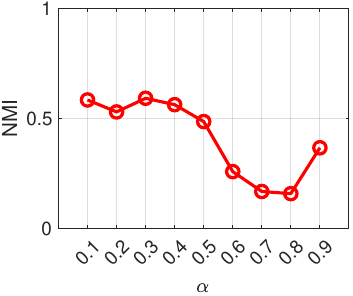}}
            \hspace{0.001\linewidth}
    \subfigure[TOX\_171]{
    \includegraphics[width=0.18\linewidth]{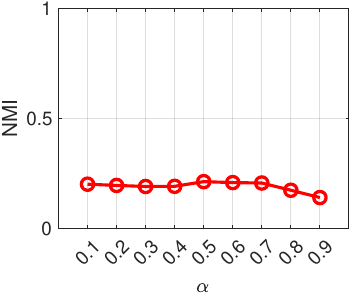}}
            \hspace{0.001\linewidth}
    \subfigure[Obesity]{
    \includegraphics[width=0.18\linewidth]{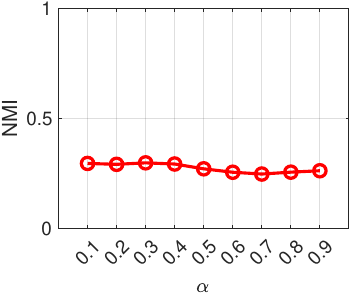}}
            \hspace{0.001\linewidth}
    \subfigure[Seeds]{
    \includegraphics[width=0.18\linewidth]{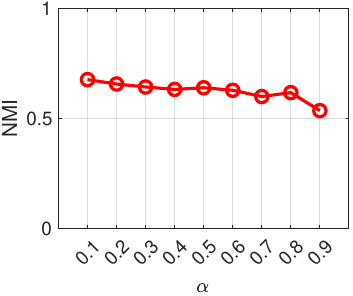}}
            \hspace{0.001\linewidth}
    \subfigure[ALLAML]{
    \includegraphics[width=0.18\linewidth]{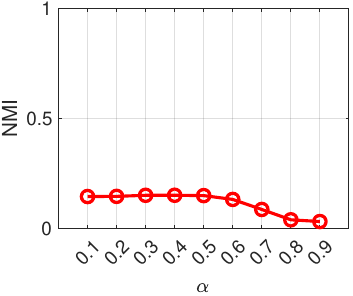}}
            \hspace{0.001\linewidth}
    \subfigure[WarpAR10P]{
    \includegraphics[width=0.18\linewidth]{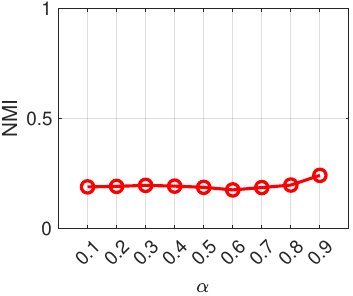}}
            \hspace{0.001\linewidth}
    \subfigure[WFRN]{
    \includegraphics[width=0.18\linewidth]{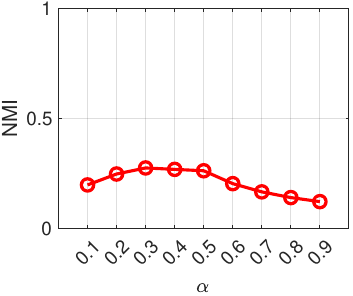}}
            \hspace{0.001\linewidth}
    \subfigure[Abalone]{
    \includegraphics[width=0.18\linewidth]{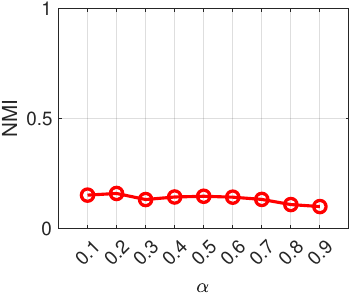}}
            \hspace{0.001\linewidth}
    \subfigure[Website]{
    \includegraphics[width=0.18\linewidth]{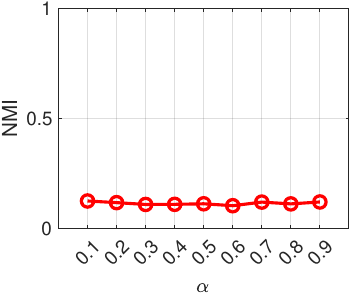}}
            \hspace{0.001\linewidth}
    \caption{Analysis of hyperparameter $\alpha$ in $\tilde{K}$. We vary the value of $\alpha$ from 0.1 to 0.9, with an incremental step of 0.1.}
    \label{para}
\end{figure*}

\begin{table}[!ht]
    \centering
    \caption{Ablation experiments (clustering performance: \%). We separately remove the Bias term (denoted as w/o Bias) and the Diversity term (denoted as w/o Diversity) from the original model to observe changes in the model's performance across three metrics.}
    \begin{tabular}{ccccccccccc}
    \toprule
        Method & D1 & D2 & D3 & D4 & D5 & D6 & D7 & D8 & D9 & D10  \\ \midrule
        \multicolumn{11}{c}{NMI}   \\ \midrule
        Proposed & \textbf{25.0}\tiny{$\pm$12} & \textbf{59.0}\tiny{$\pm$1} & \textbf{21.1}\tiny{$\pm$3} & \textbf{29.4}\tiny{$\pm$2} & \textbf{67.5}\tiny{$\pm$3} & \textbf{15.0}\tiny{$\pm$4} & \textbf{22.9}\tiny{$\pm$2} & \textbf{27.4}\tiny{$\pm$2} & \textbf{15.8}\tiny{$\pm$3} & \textbf{12.4}\tiny{$\pm$4}  \\
        w/o Bias & 8.7\tiny{$\pm$4} & 38.5\tiny{$\pm$11} & 19.3\tiny{$\pm$4} & 27.0\tiny{$\pm$2} & 59.4\tiny{$\pm$9} & 10.5\tiny{$\pm$5} & 20.0\tiny{$\pm$2} & 18.2\tiny{$\pm$3} & 15.5\tiny{$\pm$2} & 10.5\tiny{$\pm$4}  \\ 
        w/o Diversity & 8.3\tiny{$\pm$5} & 56.1\tiny{$\pm$7} & 18.9\tiny{$\pm$3} & 29.1\tiny{$\pm$3} & 62.4\tiny{$\pm$3} & 14.9\tiny{$\pm$4} & 18.2\tiny{$\pm$2} & 25.1\tiny{$\pm$2} & 14.2\tiny{$\pm$5} & 9.2\tiny{$\pm$5}  \\ \midrule
        \multicolumn{11}{c}{ARI}   \\ \midrule
        Proposed & \textbf{30.8}\tiny{$\pm$15} & \textbf{69.5}\tiny{$\pm$2} & \textbf{16.7}\tiny{$\pm$3} & \textbf{22.1}\tiny{$\pm$2} & \textbf{67.5}\tiny{$\pm$5} & \textbf{21.5}\tiny{$\pm$5} & \textbf{4.1}\tiny{$\pm$1} & \textbf{18.4}\tiny{$\pm$2} & \textbf{16.0}\tiny{$\pm$3} & \textbf{14.5}\tiny{$\pm$6}  \\
        w/o Bias & 8.8\tiny{$\pm$5} & 41.9\tiny{$\pm$10} & 14.6\tiny{$\pm$3} & 17.0\tiny{$\pm$3} & 55.5\tiny{$\pm$11} & 13.2\tiny{$\pm$9} & 3.5\tiny{$\pm$1} & 7.2\tiny{$\pm$4} & 15.7\tiny{$\pm$2} & 12.2\tiny{$\pm$5}  \\ 
        w/o Diversity & 6.9\tiny{$\pm$7} & 64.9\tiny{$\pm$12} & 15.0\tiny{$\pm$3} & 21.9\tiny{$\pm$3} & 58.0\tiny{$\pm$4} & 21.4\tiny{$\pm$5} & 2.4\tiny{$\pm$1} & 16.0\tiny{$\pm$1} & 14.1\tiny{$\pm$5} & 9.0\tiny{$\pm$6}  \\ \midrule
        \multicolumn{11}{c}{Purity}   \\ \midrule
        Proposed & \textbf{76.5}\tiny{$\pm$9} & \textbf{91.7}\tiny{$\pm$1} & \textbf{49.8}\tiny{$\pm$2} & \textbf{43.7}\tiny{$\pm$1} & \textbf{87.6}\tiny{$\pm$2} & \textbf{73.8}\tiny{$\pm$2} & \textbf{27.3}\tiny{$\pm$3} & \textbf{63.3}\tiny{$\pm$1} & \textbf{52.3}\tiny{$\pm$2} & \textbf{65.1}\tiny{$\pm$5}  \\ 
        w/o Bias & 62.9\tiny{$\pm$4} & 73.7\tiny{$\pm$7} & 47.7\tiny{$\pm$3} & 39.8\tiny{$\pm$2} & 80.6\tiny{$\pm$8} & 69.9\tiny{$\pm$4} & 23.4\tiny{$\pm$3} & 53.2\tiny{$\pm$1} & 52.2\tiny{$\pm$1} & 63.3\tiny{$\pm$4}  \\ 
        w/o Diversity & 62.1\tiny{$\pm$6} & 90.1\tiny{$\pm$4} & 48.3\tiny{$\pm$2} & 42.2\tiny{$\pm$2} & 82.8\tiny{$\pm$3} & 73.8\tiny{$\pm$2} & 21.6\tiny{$\pm$2} & 61.4\tiny{$\pm$3} & 51.4\tiny{$\pm$3} & 62.8\tiny{$\pm$7}  \\ 
    \toprule
    \end{tabular}

    \label{Ablation}
\end{table}

\subsection{Hyper-parameter Analysis}\label{HyAn}
In this paper, we have only one hyper-parameter, $\alpha$, which serves as the threshold for extracting high-confidence elements. Fig. \ref{para} shows the performance of our model under different $\alpha$ settings. It can be seen that our method is quite robust across most datasets, and the optimal hyper-parameter is generally between $0.1$ and $0.3$. From the comparative experiments, we can also see that even with fixed parameters, our algorithm performs well. Therefore, we think that our algorithm is robust to the hyper-parameter $\alpha$.

\subsection{Ablation Experiment}\label{AbEx}
Table \ref{Ablation} presents the results of our ablation experiments. Our model primarily consists of two components, and we observe the outcomes after removing each one. It is apparent that removing either component leads to varying degrees of performance degradation. When the first component, Bias, is removed, the algorithm is completely dominated by Diversity, which may cause the optimization process to deviate from the correct direction; on the other hand, removing Diversity causes the ensemble algorithm to lose its robust advantage as an ensemble method. Therefore, our algorithm derived from theoretical analysis incorporates both components, resulting in enhanced performance.

\subsection{Ensemble Size Analysis}\label{Size_Exp}
\begin{figure}
    \centering
    \includegraphics[width=1\linewidth]{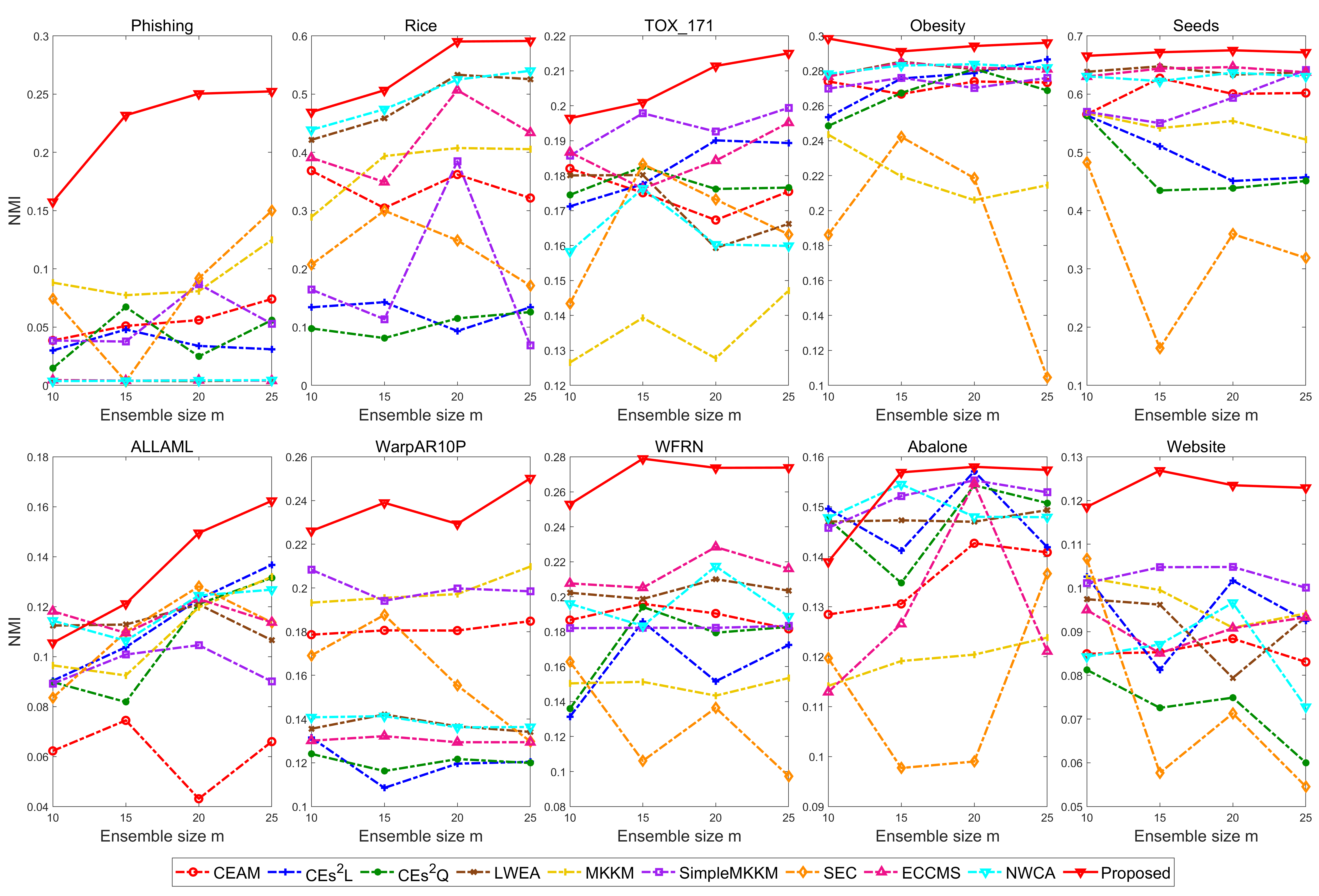}
    \caption{On each dataset, we vary the number of base clusterings $m$ in the ensemble and observe the corresponding changes in performance, as measured by NMI.}
    \label{fig:size}
\end{figure}

Figure \ref{fig:size} reports the results of all methods across different datasets by varying the ensemble size $m$ in terms of NMI. It can be observed that our method outperforms the compared SOTA methods on almost all datasets, except for the ALLAML and Abalone datasets when the ensemble size $m$ is 10. Additionally, it is evident that the performance of our method generally improves as $m$ increases, which aligns with the conclusion derived from Theorem \ref{generalization}.

%% file: main.bbl
\begin{thebibliography}{44}
\providecommand{\natexlab}[1]{#1}
\providecommand{\url}[1]{\texttt{#1}}
\expandafter\ifx\csname urlstyle\endcsname\relax
  \providecommand{\doi}[1]{doi: #1}\else
  \providecommand{\doi}{doi: \begingroup \urlstyle{rm}\Url}\fi

\bibitem[Bachem et~al.(2017)Bachem, Lucic, Hassani, and Krause]{pmlr-v70-bachem17a}
Bachem, O., Lucic, M., Hassani, S.~H., and Krause, A.
\newblock Uniform deviation bounds for k-means clustering.
\newblock In \emph{Proceedings of the 34th International Conference on Machine Learning}, volume~70, pp.\  283--291. PMLR, 06--11 Aug 2017.

\bibitem[Bang et~al.(2018)Bang, Yu, and Wu]{bang2018robust}
Bang, S., Yu, Y., and Wu, W.
\newblock Robust multiple kernel k-means clustering using min-max optimization.
\newblock \emph{arXiv preprint arXiv:1803.02458}, 2018.

\bibitem[Bengio et~al.(2004)Bengio, Delalleau, Roux, Paiement, Vincent, and Ouimet]{6788387}
Bengio, Y., Delalleau, O., Roux, N.~L., Paiement, J.-F., Vincent, P., and Ouimet, M.
\newblock Learning eigenfunctions links spectral embedding and kernel pca.
\newblock \emph{Neural Computation}, 16\penalty0 (10):\penalty0 2197--2219, 2004.

\bibitem[Bonnans \& Shapiro(1998)Bonnans and Shapiro]{doi:10.1137/S0036144596302644}
Bonnans, J.~F. and Shapiro, A.
\newblock Optimization problems with perturbations: A guided tour.
\newblock \emph{SIAM Review}, 40\penalty0 (2):\penalty0 228--264, 1998.

\bibitem[Clémençon et~al.(2008)Clémençon, Lugosi, and Vayatis]{918c02b9-94d7-38ea-9cf7-cb77ed046204}
Clémençon, S., Lugosi, G., and Vayatis, N.
\newblock Ranking and empirical minimization of u-statistics.
\newblock \emph{The Annals of Statistics}, 36\penalty0 (2):\penalty0 844--874, 2008.

\bibitem[Fern \& Brodley(2003)Fern and Brodley]{fern2003random}
Fern, X.~Z. and Brodley, C.~E.
\newblock Random projection for high dimensional data clustering: A cluster ensemble approach.
\newblock In \emph{Proceedings of the 20th international conference on machine learning (ICML-03)}, pp.\  186--193, 2003.

\bibitem[Fred \& Jain(2005)Fred and Jain]{fred2005combining}
Fred, A.~L. and Jain, A.~K.
\newblock Combining multiple clusterings using evidence accumulation.
\newblock \emph{IEEE Transactions on Pattern Analysis and Machine Intelligence}, 27\penalty0 (6):\penalty0 835--850, 2005.

\bibitem[Hadjitodorov et~al.(2006)Hadjitodorov, Kuncheva, and Todorova]{HADJITODOROV2006264}
Hadjitodorov, S.~T., Kuncheva, L.~I., and Todorova, L.~P.
\newblock Moderate diversity for better cluster ensembles.
\newblock \emph{Information Fusion}, 7\penalty0 (3):\penalty0 264--275, 2006.

\bibitem[Huang et~al.(2016)Huang, Lai, and Wang]{7337436}
Huang, D., Lai, J.-H., and Wang, C.-D.
\newblock Robust ensemble clustering using probability trajectories.
\newblock \emph{IEEE Transactions on Knowledge and Data Engineering}, 28\penalty0 (5):\penalty0 1312--1326, 2016.

\bibitem[Huang et~al.(2018)Huang, Wang, and Lai]{huang2017locally}
Huang, D., Wang, C.-D., and Lai, J.-H.
\newblock Locally weighted ensemble clustering.
\newblock \emph{IEEE Transactions on Cybernetics}, 48\penalty0 (5):\penalty0 1460--1473, 2018.

\bibitem[Huang et~al.(2021)Huang, Wang, Peng, Lai, and Kwoh]{8525437}
Huang, D., Wang, C.-D., Peng, H., Lai, J., and Kwoh, C.-K.
\newblock Enhanced ensemble clustering via fast propagation of cluster-wise similarities.
\newblock \emph{IEEE Transactions on Systems, Man, and Cybernetics: Systems}, 51\penalty0 (1):\penalty0 508--520, 2021.

\bibitem[Ji et~al.(2024)Ji, Sun, Peng, Pang, and Zhou]{ji2024clustering}
Ji, X., Sun, J., Peng, J., Pang, Y., and Zhou, P.
\newblock Clustering ensemble based on fuzzy matrix self-enhancement.
\newblock \emph{IEEE Transactions on Knowledge and Data Engineering}, 2024.

\bibitem[Jia et~al.(2011)Jia, Xiao, Liu, and Jiao]{JIA20111456}
Jia, J., Xiao, X., Liu, B., and Jiao, L.
\newblock Bagging-based spectral clustering ensemble selection.
\newblock \emph{Pattern Recognition Letters}, 32\penalty0 (10):\penalty0 1456--1467, 2011.

\bibitem[Jia et~al.(2019)Jia, Kwong, Hou, and Wu]{jia2019semi}
Jia, Y., Kwong, S., Hou, J., and Wu, W.
\newblock Semi-supervised non-negative matrix factorization with dissimilarity and similarity regularization.
\newblock \emph{IEEE transactions on neural networks and learning systems}, 31\penalty0 (7):\penalty0 2510--2521, 2019.

\bibitem[Jia et~al.(2021)Jia, Liu, Hou, Kwong, and Zhang]{jia2021multi}
Jia, Y., Liu, H., Hou, J., Kwong, S., and Zhang, Q.
\newblock Multi-view spectral clustering tailored tensor low-rank representation.
\newblock \emph{IEEE Transactions on Circuits and Systems for Video Technology}, 31\penalty0 (12):\penalty0 4784--4797, 2021.

\bibitem[Jia et~al.(2024)Jia, Tao, Wang, and Wang]{jia2023ensemble}
Jia, Y., Tao, S., Wang, R., and Wang, Y.
\newblock Ensemble clustering via co-association matrix self-enhancement.
\newblock \emph{IEEE Transactions on Neural Networks and Learning Systems}, 35\penalty0 (8):\penalty0 11168--11179, 2024.

\bibitem[Kuncheva \& Vetrov(2006)Kuncheva and Vetrov]{kuncheva2006evaluation}
Kuncheva, L.~I. and Vetrov, D.~P.
\newblock Evaluation of stability of k-means cluster ensembles with respect to random initialization.
\newblock \emph{IEEE transactions on pattern analysis and machine intelligence}, 28\penalty0 (11):\penalty0 1798--1808, 2006.

\bibitem[Latała \& Oleszkiewicz(1994)Latała and Oleszkiewicz]{Lata}
Latała, R. and Oleszkiewicz, K.
\newblock On the best constant in the khinchin-kahane inequality.
\newblock \emph{Studia Mathematica}, 109\penalty0 (1):\penalty0 101--104, 1994.

\bibitem[Li et~al.(2019)Li, Qian, Wang, Dang, and Jing]{LI201937}
Li, F., Qian, Y., Wang, J., Dang, C., and Jing, L.
\newblock Clustering ensemble based on sample's stability.
\newblock \emph{Artificial Intelligence}, 273:\penalty0 37--55, 2019.

\bibitem[Li et~al.(2023)Li, Ouyang, and Liu]{Li_Ouyang_Liu_2023}
Li, S., Ouyang, S., and Liu, Y.
\newblock Understanding the generalization performance of spectral clustering algorithms.
\newblock \emph{Proceedings of the AAAI Conference on Artificial Intelligence}, 37\penalty0 (7):\penalty0 8614--8621, Jun. 2023.

\bibitem[Li \& Jia(2025)Li and Jia]{li2025conmix}
Li, Z. and Jia, Y.
\newblock Conmix: Contrastive mixup at representation level for long-tailed deep clustering.
\newblock In \emph{The Thirteenth International Conference on Learning Representations}, 2025.

\bibitem[Liang et~al.(2022)Liang, Liu, Zhou, Liu, Wang, and Zhu]{Liang_Liu_Zhou_Liu_Wang_Zhu_2022}
Liang, W., Liu, X., Zhou, S., Liu, J., Wang, S., and Zhu, E.
\newblock Robust graph-based multi-view clustering.
\newblock \emph{Proceedings of the AAAI Conference on Artificial Intelligence}, 36\penalty0 (7):\penalty0 7462--7469, Jun. 2022.

\bibitem[Liang et~al.(2023)Liang, Liu, Liu, Ma, Zhao, Liu, and Zhu]{pmlr-v202-liang23b}
Liang, W., Liu, X., Liu, Y., Ma, C., Zhao, Y., Liu, Z., and Zhu, E.
\newblock Consistency of multiple kernel clustering.
\newblock In \emph{Proceedings of the 40th International Conference on Machine Learning}, pp.\  20650--20676, 2023.

\bibitem[Liang et~al.(2024)Liang, Tang, Liu, Liu, Liu, Zhu, and He]{10496870}
Liang, W., Tang, C., Liu, X., Liu, Y., Liu, J., Zhu, E., and He, K.
\newblock On the consistency and large-scale extension of multiple kernel clustering.
\newblock \emph{IEEE Transactions on Pattern Analysis and Machine Intelligence}, 46\penalty0 (10):\penalty0 6935--6947, 2024.

\bibitem[Liu et~al.(2017)Liu, Wu, Liu, Tao, and Fu]{7811216}
Liu, H., Wu, J., Liu, T., Tao, D., and Fu, Y.
\newblock Spectral ensemble clustering via weighted k-means: Theoretical and practical evidence.
\newblock \emph{IEEE Transactions on Knowledge and Data Engineering}, 29\penalty0 (5):\penalty0 1129--1143, 2017.

\bibitem[Liu(2023)]{9857664}
Liu, X.
\newblock Simplemkkm: Simple multiple kernel k-means.
\newblock \emph{IEEE Transactions on Pattern Analysis and Machine Intelligence}, 45\penalty0 (4):\penalty0 5174--5186, 2023.

\bibitem[McDiarmid(1989)]{McDiarmid_1989}
McDiarmid, C.
\newblock \emph{On the method of bounded differences}, pp.\  148–188.
\newblock London Mathematical Society Lecture Note Series. Cambridge University Press, 1989.

\bibitem[Metaxas et~al.(2023)Metaxas, Tzimiropoulos, and Patras]{metaxas2023divclust}
Metaxas, I.~M., Tzimiropoulos, G., and Patras, I.
\newblock Divclust: Controlling diversity in deep clustering.
\newblock In \emph{Proceedings of the IEEE/CVF Conference on Computer Vision and Pattern Recognition}, pp.\  3418--3428, 2023.

\bibitem[Peng et~al.(2023)Peng, Liu, Jia, and Hou]{peng2023egrc}
Peng, Z., Liu, H., Jia, Y., and Hou, J.
\newblock Egrc-net: Embedding-induced graph refinement clustering network.
\newblock \emph{IEEE Transactions on Image Processing}, 32:\penalty0 6457--6468, 2023.

\bibitem[Pollard(1981)]{10.1214/aos/1176345339}
Pollard, D.
\newblock {Strong Consistency of $K$-Means Clustering}.
\newblock \emph{The Annals of Statistics}, 9\penalty0 (1):\penalty0 135 -- 140, 1981.

\bibitem[Rosasco et~al.(2010)Rosasco, Belkin, and Vito]{JMLR:v11:rosasco10a}
Rosasco, L., Belkin, M., and Vito, E.~D.
\newblock On learning with integral operators.
\newblock \emph{Journal of Machine Learning Research}, 11\penalty0 (30):\penalty0 905--934, 2010.

\bibitem[Strehl \& Ghosh(2002)Strehl and Ghosh]{strehl2002cluster}
Strehl, A. and Ghosh, J.
\newblock Cluster ensembles---a knowledge reuse framework for combining multiple partitions.
\newblock \emph{Journal of machine learning research}, 3\penalty0 (Dec):\penalty0 583--617, 2002.

\bibitem[Tao et~al.(2019)Tao, Liu, Li, Wang, and Fu]{ijcai2019p494}
Tao, Z., Liu, H., Li, J., Wang, Z., and Fu, Y.
\newblock Adversarial graph embedding for ensemble clustering.
\newblock In \emph{Proceedings of the Twenty-Eighth International Joint Conference on Artificial Intelligence, {IJCAI-19}}, pp.\  3562--3568. International Joint Conferences on Artificial Intelligence Organization, 7 2019.

\bibitem[Topchy et~al.(2005)Topchy, Jain, and Punch]{topchy2005clustering}
Topchy, A., Jain, A., and Punch, W.
\newblock Clustering ensembles: models of consensus and weak partitions.
\newblock \emph{IEEE Transactions on Pattern Analysis and Machine Intelligence}, 27\penalty0 (12):\penalty0 1866--1881, 2005.

\bibitem[Vershynin(2018)]{Vershynin_2018}
Vershynin, R.
\newblock \emph{Concentration Without Independence}, pp.\  98–126.
\newblock Cambridge Series in Statistical and Probabilistic Mathematics. Cambridge University Press, 2018.

\bibitem[Von~Luxburg et~al.(2008)Von~Luxburg, Belkin, and Bousquet]{von2008consistency}
Von~Luxburg, U., Belkin, M., and Bousquet, O.
\newblock Consistency of spectral clustering.
\newblock \emph{The Annals of Statistics}, pp.\  555--586, 2008.

\bibitem[Xu et~al.(2024)Xu, Li, and Duan]{Xu_Li_Duan_2024}
Xu, J., Li, T., and Duan, L.
\newblock Enhancing ensemble clustering with adaptive high-order topological weights.
\newblock \emph{Proceedings of the AAAI Conference on Artificial Intelligence}, 38\penalty0 (14):\penalty0 16184--16192, Mar. 2024.

\bibitem[Yi et~al.(2012)Yi, Yang, Jin, Jain, and Mahdavi]{6413733}
Yi, J., Yang, T., Jin, R., Jain, A.~K., and Mahdavi, M.
\newblock Robust ensemble clustering by matrix completion.
\newblock In \emph{2012 IEEE 12th International Conference on Data Mining}, pp.\  1176--1181, 2012.

\bibitem[Yu et~al.(2014)Yu, Wang, and Samworth]{10.1093/biomet/asv008}
Yu, Y., Wang, T., and Samworth, R.~J.
\newblock A useful variant of the davis–kahan theorem for statisticians.
\newblock \emph{Biometrika}, 102\penalty0 (2):\penalty0 315--323, 04 2014.

\bibitem[Zhang(2022)]{zhang2022weighted}
Zhang, M.
\newblock Weighted clustering ensemble: A review.
\newblock \emph{Pattern Recognition}, 124:\penalty0 108428, 2022.

\bibitem[Zhang et~al.(2024)Zhang, Jia, Song, and Wang]{zhang2024similarity}
Zhang, X., Jia, Y., Song, M., and Wang, R.
\newblock Similarity and dissimilarity guided co-association matrix construction for ensemble clustering.
\newblock \emph{arXiv preprint arXiv:2411.00904}, 2024.

\bibitem[Zhang et~al.(2022)Zhang, Liang, Liu, Dai, Wang, Xu, and Zhu]{10.1145/3503161.3547917}
Zhang, Y., Liang, W., Liu, X., Dai, S., Wang, S., Xu, L., and Zhu, E.
\newblock Sample weighted multiple kernel k-means via min-max optimization.
\newblock In \emph{Proceedings of the 30th ACM International Conference on Multimedia}, MM '22, pp.\  1679–1687, New York, NY, USA, 2022. Association for Computing Machinery.

\bibitem[Zhou et~al.(2023)Zhou, Du, Liu, Ling, Ji, Li, and Shen]{zhou2023partial}
Zhou, P., Du, L., Liu, X., Ling, Z., Ji, X., Li, X., and Shen, Y.-D.
\newblock Partial clustering ensemble.
\newblock \emph{IEEE Transactions on Knowledge and Data Engineering}, 2023.

\bibitem[Zhou et~al.(2024)Zhou, Hu, Yan, and Du]{10238807}
Zhou, P., Hu, B., Yan, D., and Du, L.
\newblock Clustering ensemble via diffusion on adaptive multiplex.
\newblock \emph{IEEE Transactions on Knowledge and Data Engineering}, 36\penalty0 (4):\penalty0 1463--1474, 2024.

\end{thebibliography}
